\documentclass[conference]{IEEEtran}
\IEEEoverridecommandlockouts
\usepackage{cite}
\usepackage{amsmath,amssymb,amsfonts}
\usepackage{graphicx}
\usepackage{textcomp}
\usepackage{xcolor}
\def\BibTeX{{\rm B\kern-.05em{\sc i\kern-.025em b}\kern-.08em
    T\kern-.1667em\lower.7ex\hbox{E}\kern-.125emX}}

\usepackage[]{hyperref}
\usepackage{subcaption}
\usepackage[compatibility=false]{caption}
\usepackage{tikz}
\usepackage{tabularx}
\usepackage{algorithm}
\usepackage{algpseudocode}
\usepackage{listings}
\usepackage{xcolor}
\usepackage{pl-syntax}
\usepackage{booktabs}
\usepackage{pifont}
\usepackage{tikz-cd} 
\usepackage{enumitem}
\usepackage{ebproof}
\usepackage{makecell}
\usepackage{array}
\usepackage{amsthm}
\usepackage{xspace}
\usepackage[online,flushleft]{threeparttable}
\usepackage{mathtools} 

\newcommand{\code}[1]{\texttt{#1}}

\usepackage[switch]{lineno}
\usepackage{multicol}
\usepackage{microtype}
\usepackage{siunitx}
\usepackage{url}

\usepackage{transparent}
\usepackage{eso-pic}

\NewDocumentCommand{\circleddefaults}{
    O{black}
    O{black}
    O{white}
    m
    }{
    \tikz[baseline=(char.base)]{
    \node[shape=circle, fill=#1, draw=#2, text=#3, inner sep=0.5pt] (char) {\small#4};}}
\usepackage{enumitem}
\setlist{nolistsep}

\makeatletter
\newcommand{\thickhline}{%
    \noalign {\ifnum 0=`}\fi \hrule height 1pt
    \futurelet \reserved@a \@xhline
}
\newcolumntype{"}{@{\hskip\tabcolsep\vrule width 1pt\hskip\tabcolsep}}
\makeatother

\definecolor{codegreen}{rgb}{0,0.6,0}
\definecolor{codegray}{rgb}{0.5,0.5,0.5}
\definecolor{codepurple}{rgb}{0.58,0,0.82}
\definecolor{backcolour}{rgb}{0.95,0.95,0.92}

\lstdefinestyle{mystyle}{
    backgroundcolor=\color{backcolour},   
    commentstyle=\color{codegreen},
    keywordstyle=\color{magenta},
    numberstyle=\small\color{codegray},
    stringstyle=\color{codepurple},
    basicstyle=\ttfamily,
    breakatwhitespace=false,         
    breaklines=true,                 
    captionpos=b,                    
    keepspaces=true,                 
    numbers=left,                    
    numbersep=5pt,                  
    showspaces=false,                
    showstringspaces=false,
    showtabs=false,                  
    tabsize=2
}

\newtheorem{theorem}{Theorem}

\usepackage{microtype}
\usepackage{booktabs} 
\usepackage{mathrsfs}

\usepackage{multirow, makecell}
\usepackage{minted}
\usepackage[utf8]{inputenc}
\usepackage[T1]{fontenc}
\usepackage{microtype}
\makeatletter
\newcommand\fs@betterruled{%
  \def\@fs@cfont{\bfseries}\let\@fs@capt\floatc@ruled
  \def\@fs@pre{\vspace*{5pt}\hrule height.8pt depth0pt \kern2pt}%
  \def\@fs@post{\kern2pt\hrule\relax}%
  \def\@fs@mid{\kern2pt\hrule\kern2pt}%
  \let\@fs@iftopcapt\iftrue}
\floatstyle{betterruled}
\restylefloat{algorithm}
\makeatother

\begin{document}
\pagenumbering{gobble}

\title{Hexcute: A Compiler Framework for Automating Layout Synthesis in GPU Programs}


\author{\IEEEauthorblockN{Xiao Zhang}
\IEEEauthorblockA{\textit{University of Toronto}\\
\textit{NVIDIA}\\
Toronto, Canada
\\ zita.zhang@mail.utoronto.ca
}
\\
\IEEEauthorblockN{Yang Hu}
\IEEEauthorblockA{\textit{NVIDIA} \\
Toronto, Canada
\\ maxhu@nvidia.com
}
\and
\IEEEauthorblockN{Yaoyao Ding}
\IEEEauthorblockA{\textit{University of Toronto}\\
\textit{NVIDIA}\\
\textit{Vector Institute}\\
Toronto, Canada
\\ yaoyao@cs.toronto.edu
}
\\
\IEEEauthorblockN{Tatiana Shpeisman}
\IEEEauthorblockA{\textit{NVIDIA} \\
Toronto, Canada 
\\ tshpeisman@nvidia.com
}
\and
\IEEEauthorblockN{Bolin Sun}
\IEEEauthorblockA{\textit{University of Toronto}\\
\textit{NVIDIA}\\
Toronto, Canada 
\\ bolin.sun@mail.utoronto.ca
}
\\
\IEEEauthorblockN{Gennady Pekhimenko}
\IEEEauthorblockA{\textit{University of Toronto}\\
\textit{NVIDIA}\\
\textit{Vector Institute}\\
Toronto, Canada
\\ pekhimenko@cs.toronto.edu
}
}
\IEEEaftertitletext{\vspace{-1\baselineskip}} 

\maketitle
\thispagestyle{plain}
\pagestyle{plain}

\begin{abstract} Efficient GPU programming is crucial for achieving high performance in deep learning (DL) applications. The performance of GPU programs depends on how data is parallelized across threads and arranged within memory subsystems. The mapping functions describing tensors on GPUs are known as \emph{tensor layouts}. Low-level programming frameworks, such as CUTLASS and Hidet, provide expressive layout abstractions but often require \emph{considerable programming effort} to manually specify optimal layouts. High-level GPU programming languages, such as Triton, rely on compiler heuristics to generate dataflow, layouts, and pipelining strategies in GPU programs. However, the heuristics for dataflow and pipelining strategies are not generalizable to complex operators.  To balance expressiveness and programmability, we propose Hexcute, a compiler framework that automates layout synthesis while providing explicit control over dataflow and pipelining. Hexcute formalizes layout synthesis as a constraint programming problem and solves it with a type-inference-based algorithm. This approach enables systematic exploration of optimal layouts and instructions. 

Our evaluation shows that Hexcute matches the performance of libraries like cuBLAS and FlashAttention on GEMM, Attention, and their variants, while reducing the amount of code by 1.27$\times$-7.94$\times$ compared to CUTLASS. For mixed-type mixture-of-experts (MoE) operators, Hexcute achieves an average speedup of 6.46$\times$ over Triton. In the end-to-end evaluations of vLLM, Hexcute delivers up to 2.60$\times$ speedup on DeepSeek-R1-AWQ and 2.04$\times$ on a Mamba-based model.
\end{abstract}


\section{Introduction}
Deep learning (DL) has achieved breakthroughs across a wide range of applications, such as image recognition~\cite{He_2016_CVPR,diffuser,fastrcnn}, natural language processing~\cite{llama,multitask}, autonomous driving~\cite{yurtsever2020survey,autonomousdriving}, and drug discovery~\cite{drug}. Recent large language models (LLMs)\cite{lang_few_shot} have further broadened DL’s impact. Deep neural networks (DNNs) rely on computationally intensive tensor operations and therefore require specialized hardware accelerators such as GPUs~\cite{gpu,amdgpu}, TPUs~\cite{tpu}, and NPUs~\cite{npu,huaweinpu}. Among these, GPUs remain the most accessible accelerators for both research and industry.

Efficient GPU programming is challenging due to complex architectures and execution models. Modern GPUs feature specialized units, such as NVIDIA’s Tensor Cores and AMD’s Matrix Cores, which execute instructions collectively rather than following the traditional SIMT model~\cite{gpu_2}. The GPU memory hierarchy comprises global memory (DRAM), a shared L2 cache, software-managed shared memory, and fast register memory managed by threads. Memory access on GPUs requires careful optimization. For example, global memory accesses must be coalesced~\cite{davidson1994memory} and shared memory accesses must avoid bank conflicts~\cite{tretter2017minimising}. The way tensors are distributed among thread registers and organized in memory, referred to as \emph{tensor layout}~\cite{Triton}, has a critical impact on performance. Poor layouts can lead to uncoalesced memory accesses or shared-memory bank conflicts, which limit performance.

Low-level programming frameworks, such as CUTLASS~\cite{Thakkar_CUTLASS_2023}, Graphene~\cite{graphene}, and Hidet~\cite{Hidet}, provide expressive abstractions for constructing layouts and mapping threads to data. These systems enable programmers to explicitly specify complex mappings and support optimizations such as bank conflict avoidance, optimal dataflow, and efficient pipelining. However, users must manually specify layouts and ensure they remain consistent across tensor operations and hardware instructions. This process is both error-prone and time-consuming, requiring deep hardware expertise. 

High-level programming models, such as Triton, automatically generate dataflow and tensor layouts, as well as pipelining strategies. While this approach reduces programming effort, it fails to generalize to emerging operators, such as mixed-type operators for weight-only quantization~\cite{AWQ,frantar2023gptq} and scan operators in the Mamba architecture~\cite{mamba}. First, the heuristics for dataflow and pipelining are not generalizable. Triton sometimes places tensors in suboptimal memory, which causes excessive copies across memory hierarchies. Users also cannot control tensor placement. Moreover, Triton adopts a case-by-case layout system, which cannot systematically infer the complex layouts needed for mixed-type operations.

To balance programmability and explicit control, we propose Hexcute, a compiler framework that automates layout synthesis in GPU programs. Hexcute is built on top of a tile-level programming interface with shared memory and register abstractions, which allows explicit expression of optimizations, including dataflow and pipelining. Hexcute provides a principled, general framework for layout synthesis and collective instruction selection. Hexcute (1) models layouts as functions, (2) embeds them in tensor types, (3) builds constraints that respect hardware instruction requirements using algebraic operations, and (4) solves them with a type-inference-based algorithm~\cite{type_infer}. 

We implement Hexcute and integrate Hexcute kernels into the state-of-the-art (SOTA) inference engine called vLLM~\cite{vLLM}. Our evaluation shows that Hexcute matches the performance of libraries across various operators while reducing the amount of code by 1.27$\times$-7.94$\times$. For the mixed-type MoE~\cite{MoE} operator, Hexcute achieves an average speedup of 6.46$\times$ over Triton. In the end-to-end evaluations of vLLM, Hexcute delivers up to 2.60$\times$ speedup on DeepSeek-R1-AWQ and 2.04$\times$ on a Mamba-based model. We summarize our contributions as follows:

\begin{itemize}
\item To balance programmability and explicit control, we propose a compiler framework called Hexcute for automating layout synthesis. It is built on top of a tile-based programming model with shared memory and register abstractions, allowing for the explicit expression of optimizations such as dataflow and pipelining strategies. 
\item To synthesize complex layouts in GPU programs, we propose a type-inference-based approach that ensures correctness by construction and achieves high performance through systematic instruction selection. 
\item We demonstrate that Hexcute matches the performance of libraries while reducing the amount of code by 1.27$\times$-7.94$\times$ compared to CUTLASS. For mixed-type MoE operators, Hexcute achieves an average speedup of 6.46$\times$ over Triton. In the end-to-end evaluations of vLLM, Hexcute delivers up to 2.60$\times$ speedup. 
\end{itemize}
\begin{figure}[t]
     \begin{subfigure}[b]{0.48\linewidth}
        \centering
        \includegraphics[width=\linewidth]{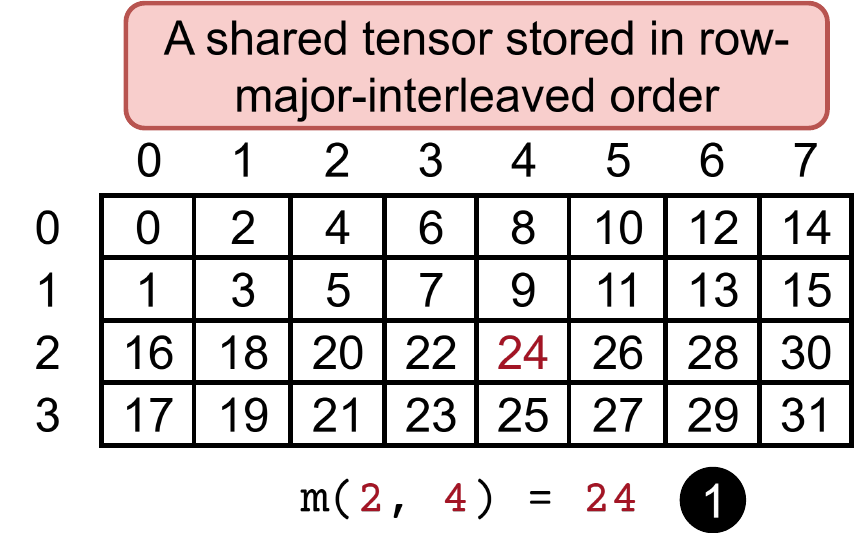}
        \caption{A shared tensor of shape 4$\times$8}
        \includegraphics[width=\linewidth]{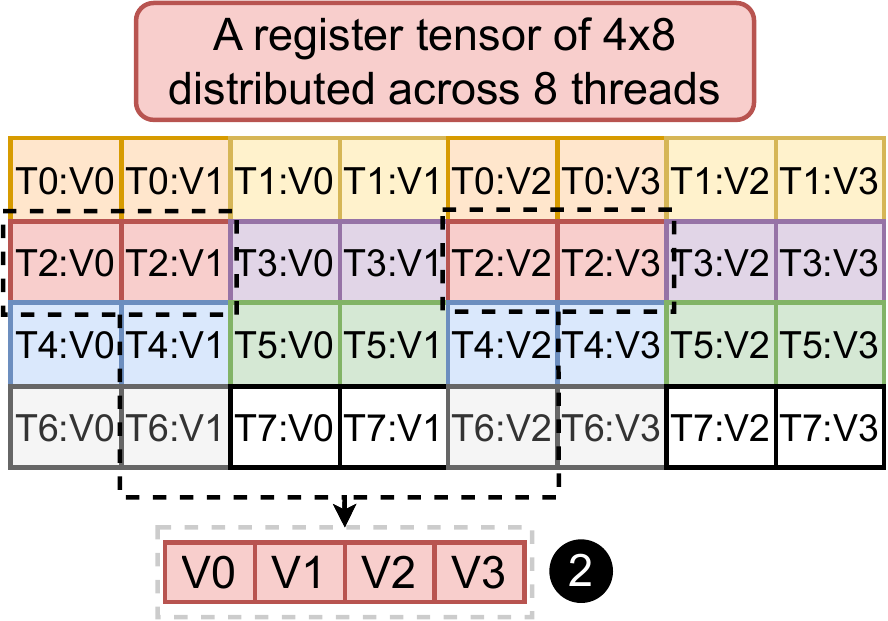}
        \caption{A register tensor of shape 4$\times$8}
    \end{subfigure}
    \begin{subfigure}[b]{0.49\linewidth}
        \centering
        \includegraphics[width=\linewidth]{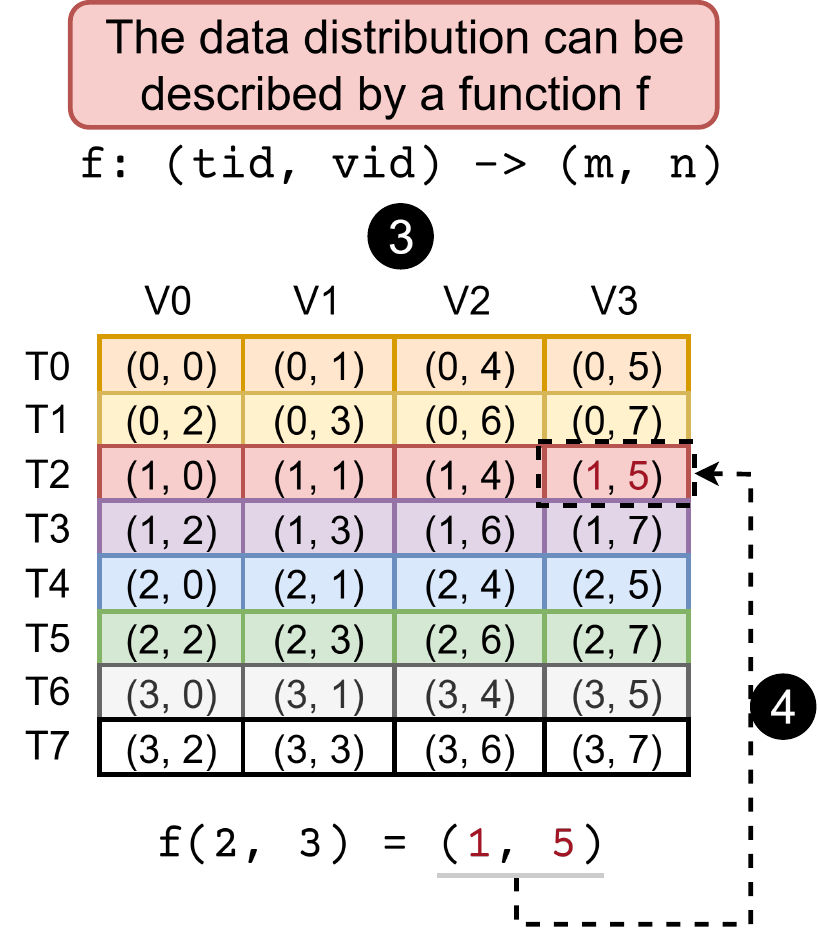}
        \caption{The distribution of the tensor in (b) can be represented with a function \texttt{f}.}
    \end{subfigure}
\caption{Shared memory and register tensors in GPU programs. (a) A shared memory tensor stored in row-major-interleaved order. (b) A register tensor distributed across 8 threads, with each thread holding 4 elements. (c) The distribution of the register tensor in (b) is described by a function \texttt{f}.} 
\label{fig:layout}
\end{figure}
\begin{figure}[t!]
    \begin{subfigure}[b]{0.59\linewidth}
        \centering
        \includegraphics[width=\textwidth]{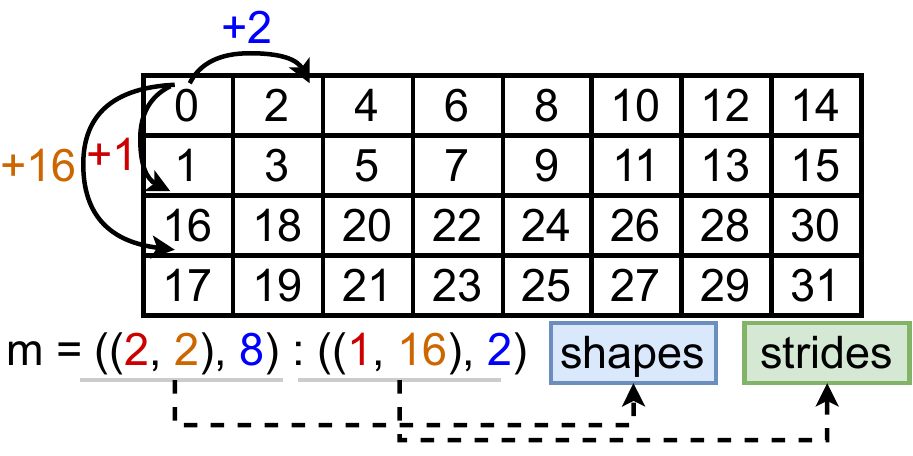}
        \caption{Row-major-interleaved layout \texttt{A}}      \includegraphics[width=\textwidth]{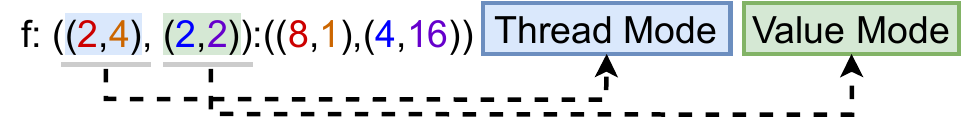}
        \caption{Thread-value layout \texttt{f}}
    \end{subfigure}
    \begin{subfigure}[b]{0.4\linewidth}
        \centering
        \includegraphics[width=\textwidth]{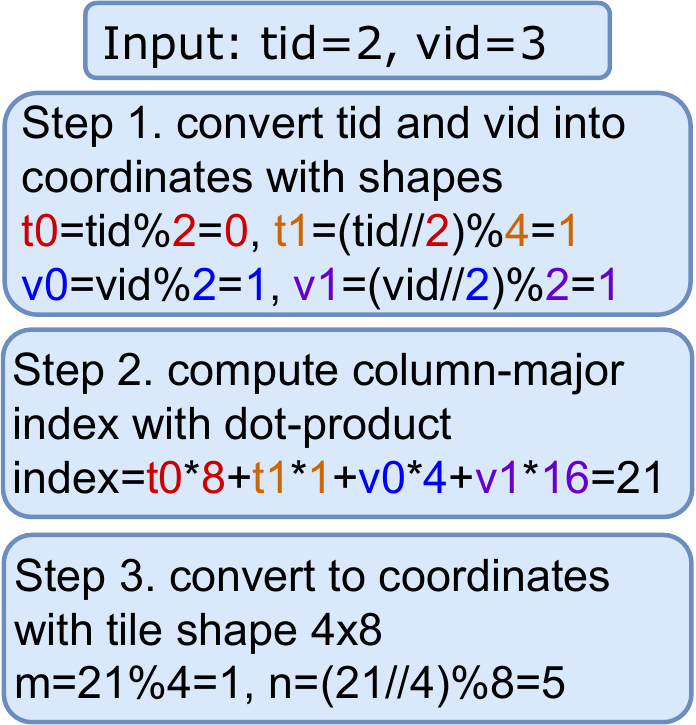}
        \caption*{(c) Compute \texttt{f(2,3)}}
    \end{subfigure}
\caption{Layout concept in CuTe. (a) Layout of Fig.~\ref{fig:layout} (a) represented as  \texttt{m=((2,2),8):((1,16),2)}. (b) Layout of Fig.~\ref{fig:layout} (b) represented as thread-value layout \texttt{f}. (c) Illustration of how \texttt{f} maps the pair $(t,v)=(2,3)$ to the position $(1,5)$ in the 4$\times$8 tensor. Colors highlight the correspondence between coordinates, shapes, and strides in (a), (b), and (c).}
\vspace{-3mm}
\label{fig:memory_layout}
\end{figure}
\section{Background and Motivation}
\subsection{Layout in GPU Programming Model}
GPU programs, known as GPU kernels~\cite{gpu}, define the behavior of thread blocks. Each thread block typically performs identical computations on statically shaped tensors~\cite{cypress}, called tiles~\cite{Triton}. These tensors can reside either in shared memory or register files. A \emph{shared memory tensor} resides in the shared memory subsystem. Fig.~\ref{fig:layout} (a) shows a 4$\times$8 shared tensor stored in row-major-interleaved order. Its layout is defined by a function \texttt{m} mapping tensor coordinates to memory addresses. The numbers in the boxes indicate outputs of \texttt{m}; for example, \circleddefaults{1} shows \texttt{m} maps coordinates \texttt{(2,4)} to memory address 24. A \emph{register tensor} is distributed across the threads of a thread block, with each thread holding a local array. Fig.~\ref{fig:layout}~(b) illustrates a 4$\times$8 register tensor distributed across 8 threads. Each thread holds a local array of 4 elements; for example, the array held by the thread with \texttt{tid=2} is highlighted by \circleddefaults{2} in Fig.~\ref{fig:layout} (b). The distribution of this tensor is described by a function \texttt{f} shown in Fig.~\ref{fig:layout} (c), where the coordinates in the boxes represent outputs of \texttt{f}. As illustrated by \circleddefaults{3} and \circleddefaults{4}, \texttt{f} maps a pair consisting of thread index \texttt{tid} and local array index \texttt{vid} to logical coordinates \texttt{(m, n)} within the tensor; for example, \texttt{(tid, vid)=(2, 3)} is mapped to coordinates \texttt{(1, 5)}.

These mapping functions, such as \texttt{m} and \texttt{f}, are known as  \emph{layouts}~\cite{linearlayouts}. Layouts are critical to the performance of GPU programs, as they determine whether memory accesses are coalesced~\cite{davidson1994memory} or free of bank conflicts~\cite{tretter2017minimising}. We refer to the problem of determining optimal layouts as the layout synthesis problem. Solving the layout synthesis problem is challenging for several reasons. First, since layouts are functions rather than integer-enumerable values, they create a vast search space that is difficult to explore exhaustively. Second, heuristic-based layout assignment lacks generality; new deep learning operators often require new heuristics~\cite{linearlayouts}. Finally, traditional approaches, such as integer linear programming~\cite{layoutselectionpbqp}, are ineffective because layouts cannot be represented as a finite set of integer variables.

\subsection{Programming Models with Manual Layout Specification}
Because layouts are essential yet difficult to derive, low-level frameworks, such as CUTLASS\cite{Thakkar_CUTLASS_2023}, Graphene~\cite{graphene}, and Hidet~\cite{Hidet}, require programmers to specify them manually. Fig.~\ref{fig:low-level_PL} illustrates an example of these programming models, where programmers first define layouts and then implement GPU kernels. Hidet~\cite{Hidet} abstracts layouts as task mappings and exposes primitives to construct task-mapping functions, while CuTe~\cite{Thakkar_CUTLASS_2023} expresses layouts as functions mapping integers to integers. In CuTe~\cite{Thakkar_CUTLASS_2023}, the \emph{Layout} is expressed as a pair of integer tuples: the \emph{shape}, which specifies the tensor dimensions, and the \emph{strides}, which map coordinates to memory addresses, written in the form (\texttt{shapes: strides}). 
CuTe~\cite{Thakkar_CUTLASS_2023} generalizes this notion to hierarchical layouts, where parentheses group dimensions into sub-layouts, called \emph{modes}. 
For example, the row-major interleaved layout \texttt{m} in Fig.~\ref{fig:layout}~(a) can be expressed using nested tuples as shown in Fig.~\ref{fig:memory_layout}~(a), and the sub-layout \texttt{(2,2):(1,16)} is a mode of layout \texttt{m}.
CuTe~\cite{Thakkar_CUTLASS_2023} also introduces the concept of \emph{thread-value layout} (TV layout) to represent the layouts of register tensors. A TV layout consists of two modes: a thread mode and a value mode. Together, they define a function mapping a thread index and a local array index to a position within a tile. For example, function \texttt{f} shown in Fig.~\ref{fig:layout} (c) can be expressed as \texttt{f=((2,4),(2,2)):((8,1),(4,16))}, illustrated in Fig.~\ref{fig:memory_layout}~(b). Fig.~\ref{fig:memory_layout}~(c) illustrates how this layout maps the pair \((t, v)=(2, 3)\) to the coordinates \((1, 5)\) in the 4$\times$8 tile.
\begin{figure}[t]
    \centering
    \includegraphics[width=\linewidth]{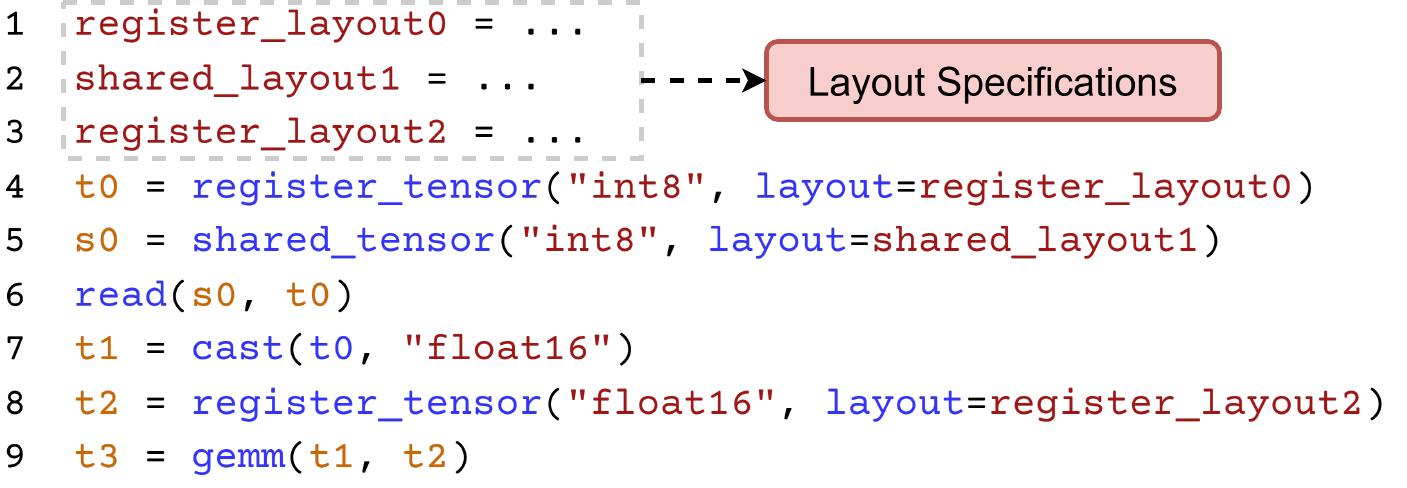}
\caption{Typical workflow in low-level programming models (e.g., Hidet, CUTLASS), where layouts are manually defined before kernel implementation.}
\vspace{-3mm}
\label{fig:low-level_PL}
\end{figure}

These programming models provide users with greater control over GPU programming, enabling the expression of optimizations such as efficient dataflow and software pipelining. However, these approaches lack an automated mechanism for synthesizing the optimal layouts. Manually specifying these layouts is challenging, as they are complex mapping functions. Additionally, layouts and operations have interdependencies. In the example of Fig.~\ref{fig:low-level_PL}, the register tensor \texttt{t0} in \textbf{Line} 4 is cast to \texttt{float16} type and then consumed by the \texttt{gemm} operation in \textbf{Line} 9. As a result, \texttt{register\_layout0} must conform to the requirements of the Tensor Core instruction used in the \texttt{gemm} operation, which imposes specific constraints on how data is distributed across threads. Tracking these dependencies during kernel development is challenging, and inconsistent layouts can result in incorrect programs. Finally, layouts are interrelated and impact performance together. For instance, \textbf{Line} 6 loads data from a shared memory tensor to a register tensor. The layouts \texttt{shared\_layout1} and \texttt{register\_layout0} determine whether vectorized instructions can be used, and whether the memory access pattern avoids bank conflicts. These complexities highlight the need for a programming model that reduces the manual effort required for layout specification.

\begin{figure}[t]
    \begin{subfigure}[t]{\linewidth}
        \centering
        \includegraphics[width=\linewidth]{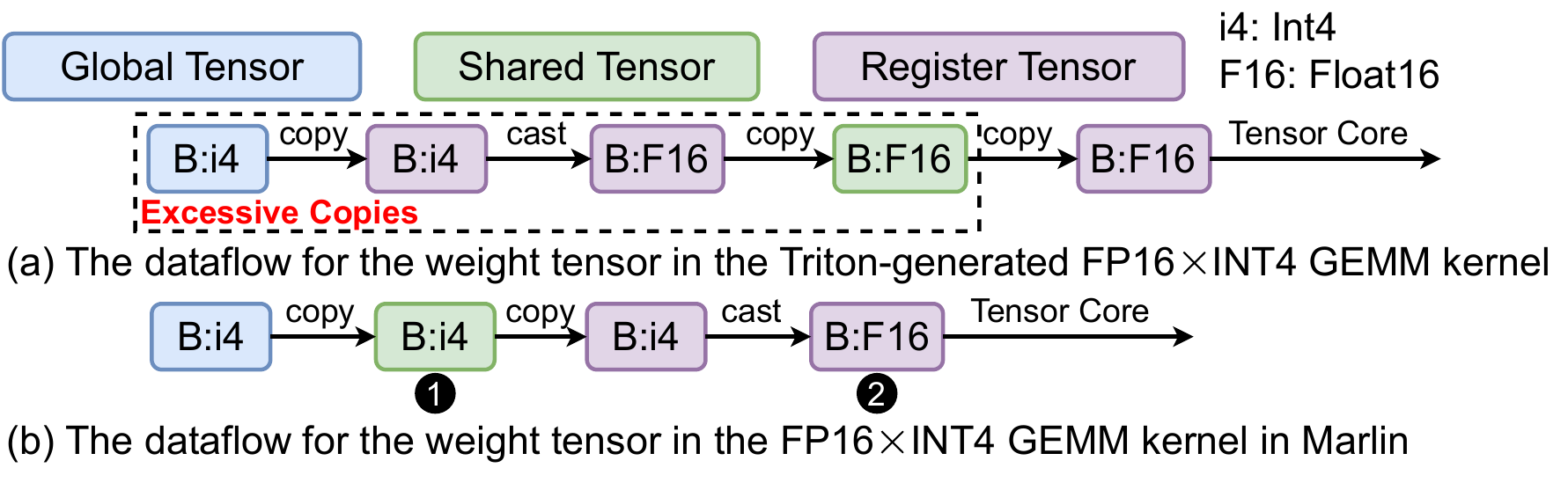}
    \end{subfigure}
\caption{Comparison of dataflows in FP16$\times$INT4 GEMM kernels implemented by Triton and Marlin.}
\vspace{-3mm}
\label{fig:dataflow_mixed_gemm}
\end{figure}
\begin{figure}[t]
    \begin{subfigure}[t]{\linewidth}
        \centering
        \includegraphics[width=\textwidth]{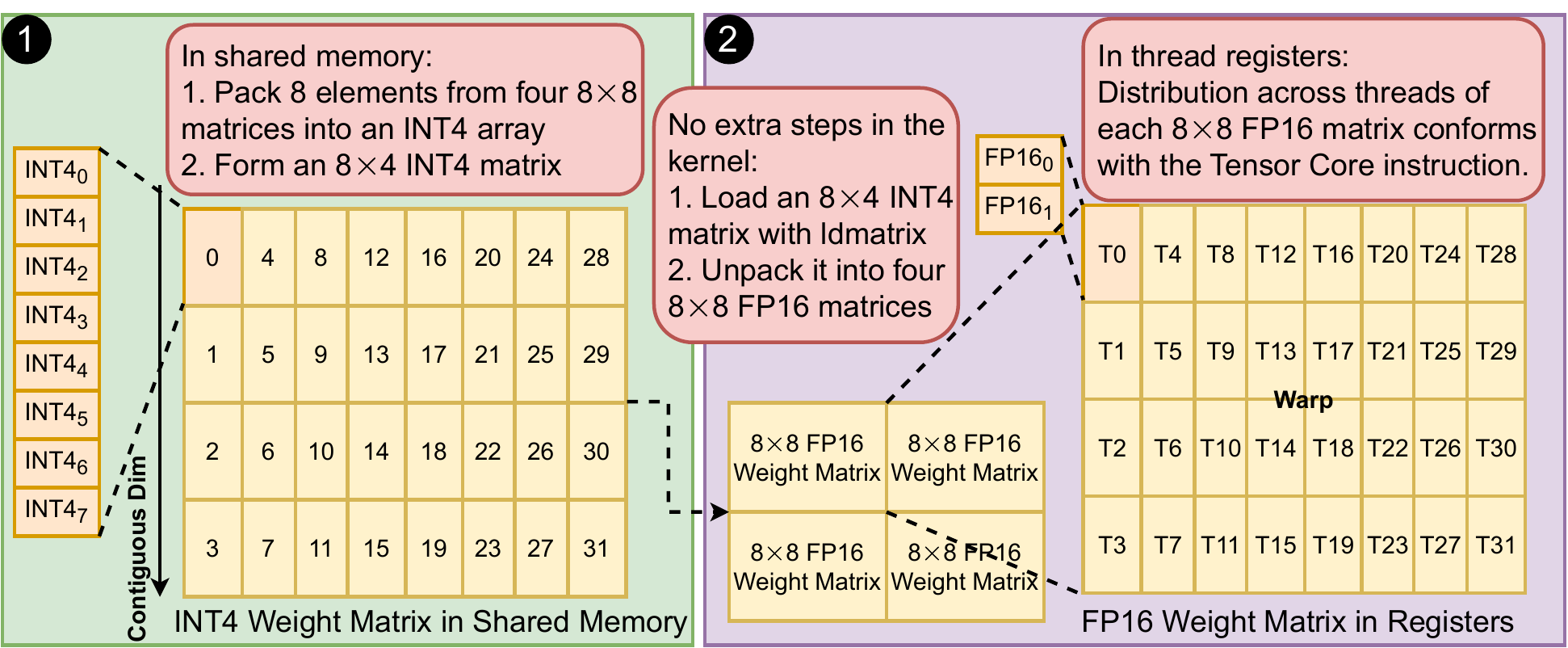}
    \end{subfigure}
\caption{The layouts of shared memory tensor \circleddefaults{1} and register tensor \circleddefaults{2}  in the efficient dataflow of  Fig.~\ref{fig:dataflow_mixed_gemm} (b). These layouts enable loading INT4 weights from shared memory with the \texttt{ldmatrix} instruction and casting INT4 weights to FP16 without additional inter-thread data exchange.}
\vspace{-3mm}
\label{fig:int4_packed_layout}
\end{figure}

\begin{figure*}[t!]
    \begin{subfigure}[b]{0.49\linewidth}
        \centering
        \includegraphics[width=\textwidth]{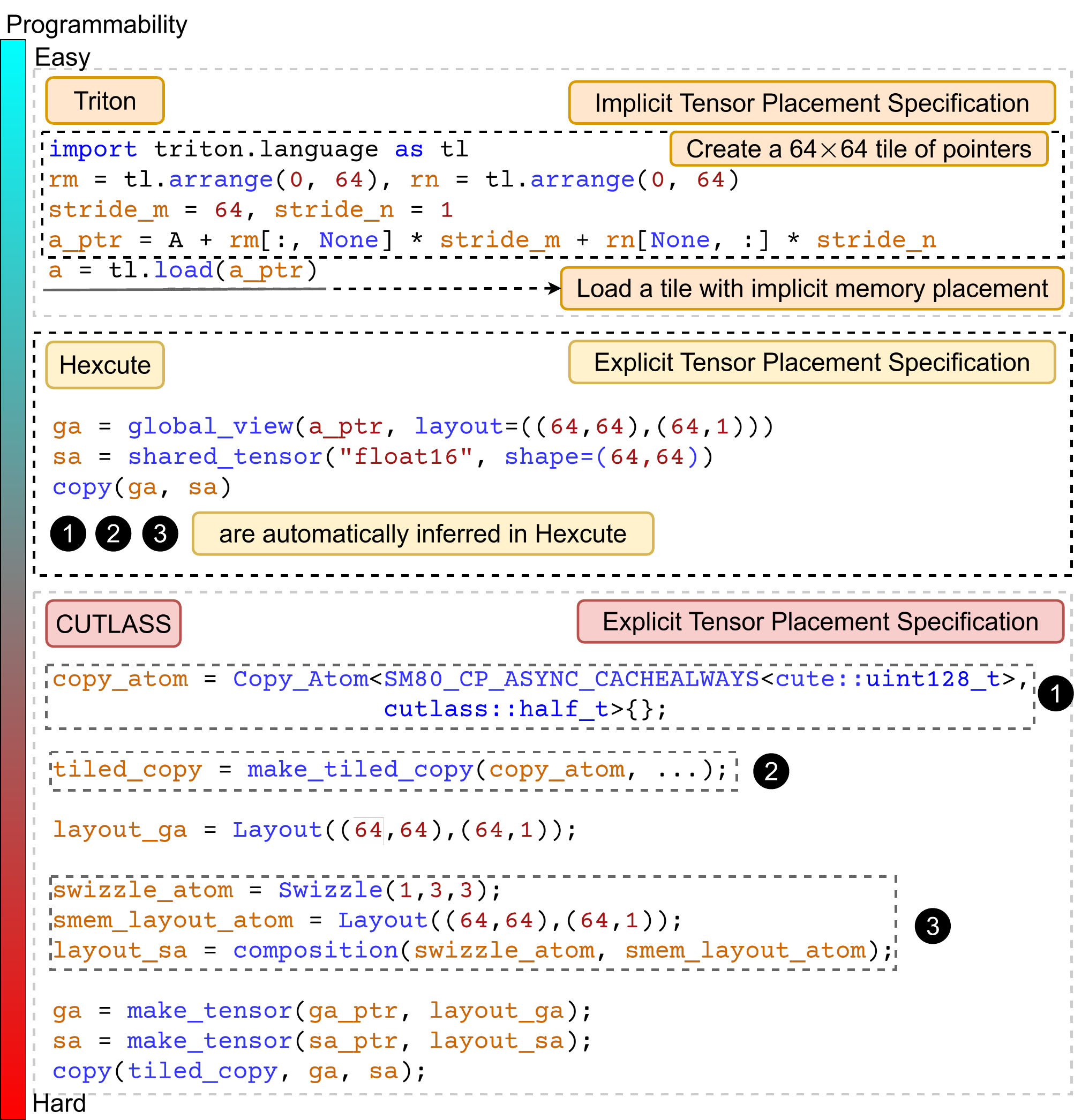}
        \caption{Copy a tensor of shape 64$\times$64 in Triton, Hexcute, and CUTLASS}
    \end{subfigure}
    \begin{subfigure}[b]{0.51\linewidth}
    \centering
    \includegraphics[width=\textwidth]{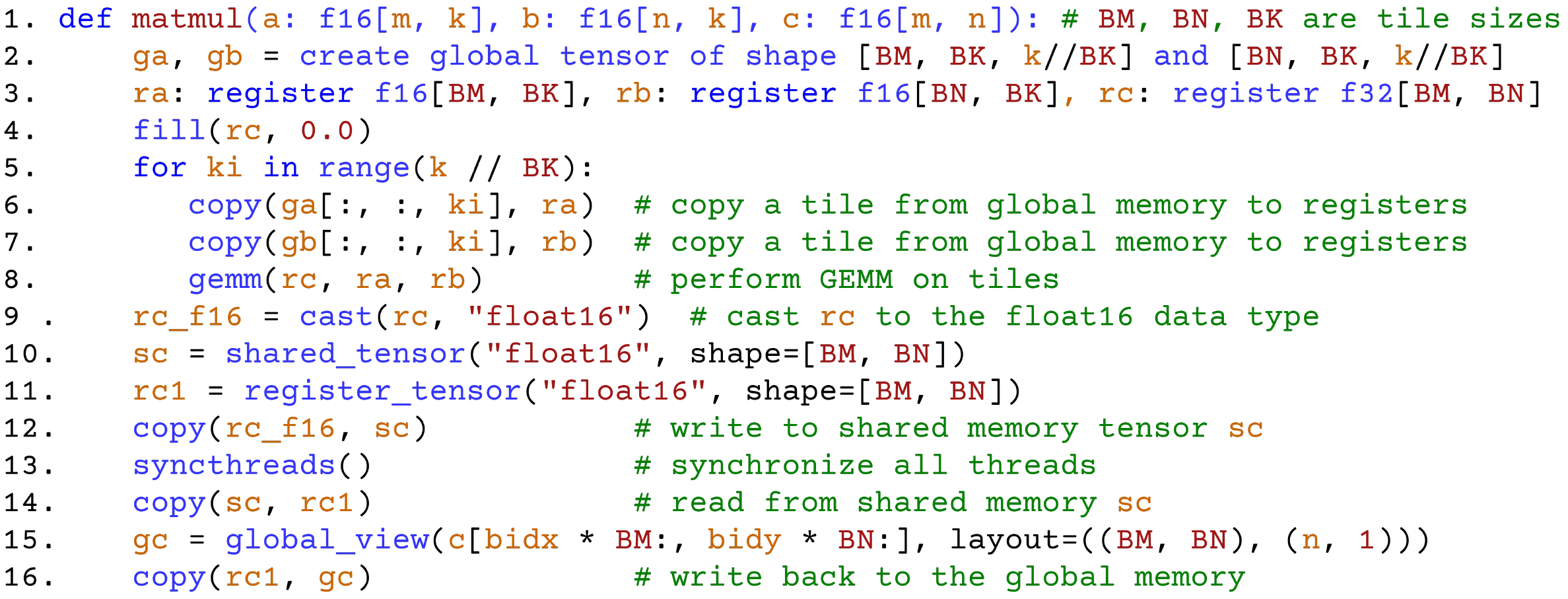}
    \caption{A GEMM kernel in Hexcute}
    \includegraphics[width=\textwidth]{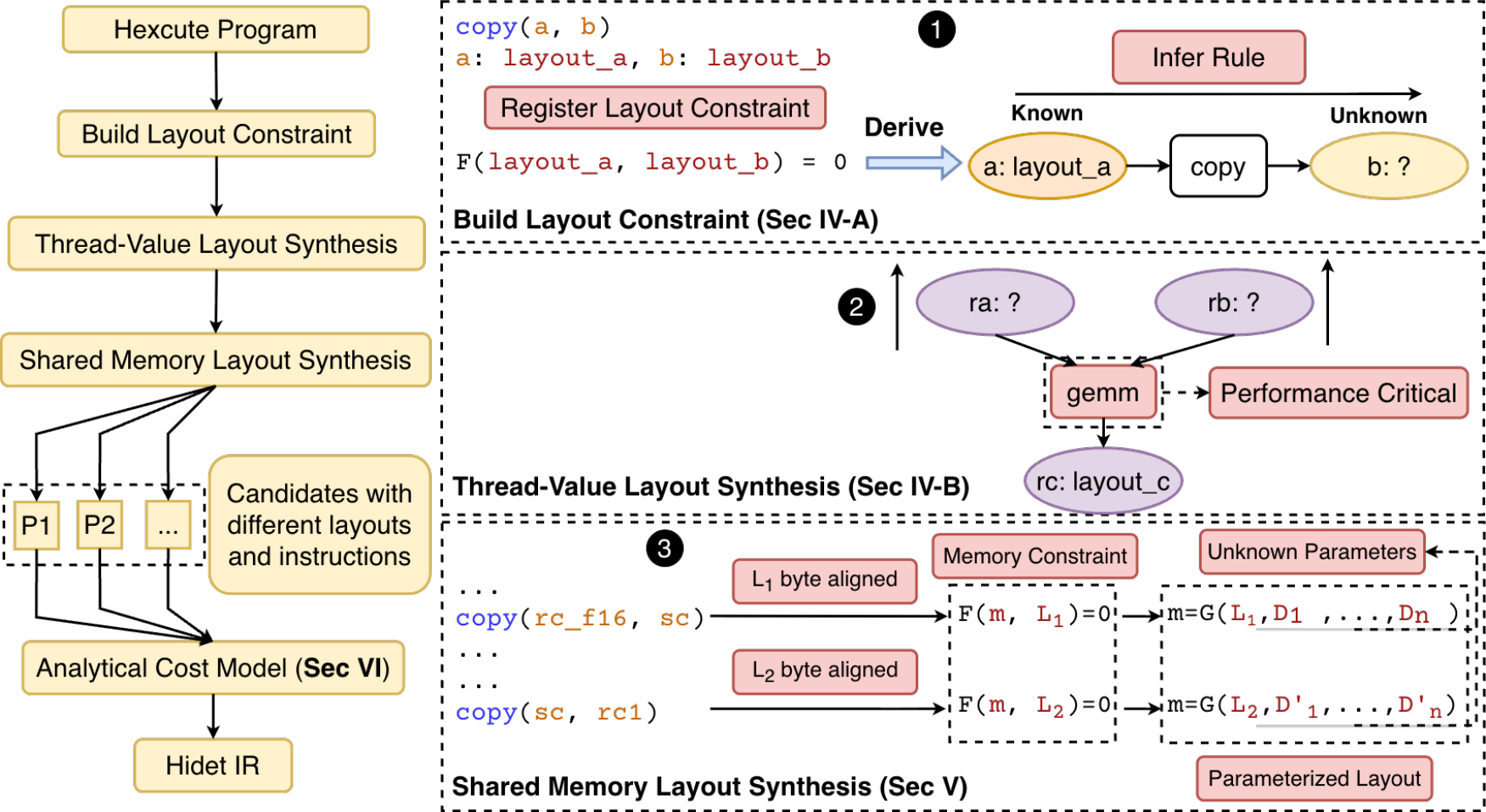}
    \caption{Compilation Workflow}
    \end{subfigure}
\caption{(a) Comparison between Hexcute, Triton, and CUTLASS. Hexcute allows explicit tensor placement, which Triton does not. CUTLASS requires users to \circleddefaults{1} select the instruction, \circleddefaults{2} map threads to their tasks, and \circleddefaults{3} specify layouts, whereas Hexcute automates all three; (b) A GEMM kernel in Hexcute; (c) The compilation workflow of Hexcute. }
\vspace{-3mm}
\label{fig:hexcute_overview}
\end{figure*}
\subsection{Programming Models with Automatic Layout Synthesis}
High-level programming models, such as Triton~\cite{Triton}, automatically generate dataflows and layouts for tensors. However, the heuristics for dataflow and pipelining strategies cannot be generalized to complex deep learning operators, such as mixed-type operations. For example, Fig.~\ref{fig:dataflow_mixed_gemm} (a) illustrates Triton's dataflow in an FP16$\times$INT4 GEMM kernel, where the activation tensor uses the FP16 data type and the weights use INT4. The dataflow is inefficient because Triton places the weight tensor in suboptimal memory spaces, causing excessive data movement between registers and shared memory (highlighted by the dashed-line box). In contrast, Fig.~\ref{fig:dataflow_mixed_gemm} (b) shows the optimized dataflow implemented in the hand-written kernel library called Marlin~\cite{frantar2024marlin}, which avoids the excessive copies. Moreover, the efficient dataflow requires shared memory and register layouts that enable loading INT4 weights from shared memory to registers and converting them to FP16 without inter-thread data exchange. The layouts, illustrated in Fig.~\ref{fig:int4_packed_layout}, of shared tensor \circleddefaults{1} and register tensor \circleddefaults{2} are complex mapping functions that cannot be handled by Triton systematically. Addressing these limitations in Triton is hard: it would require (1) a fundamental refactor of the programming interface to allow more explicit control, and (2) a synthesis algorithm that treats tensor layouts as general functions to infer the layouts in Fig.~\ref{fig:int4_packed_layout}. These limitations become severe on the H100 architecture for the MoE layer~\cite{MoE}, which is critical for SOTA models like DeepSeek-R1~\cite{deepseekr1}. Our evaluation in Section~\ref{evl:mixed_type_moe} shows that Triton’s approach can incur a latency increase of up to 10$\times$. These shortcomings highlight the need for a high-level programming model with more explicit control.

\section{Key Ideas}
\subsubsection*{Explicit Tile-level Programming Model}
To bridge the gap between high-level programming models and low-level frameworks, we propose Hexcute, a compiler framework for automating layout synthesis in GPU programs. Hexcute is built on a Python-embedded domain-specific language (DSL). Fig.~\ref{fig:hexcute_overview} (a) compares Hexcute with Triton and CUTLASS. While Triton targets novice users who primarily focus on prototyping new algorithms, Hexcute targets experienced kernel engineers who require low-level control over dataflow and pipelining and aim to deliver close-to-peak performance with minimal programming effort. Like Triton, Hexcute adopts the tile-level programming model. It extends Hidet Script~\cite{hidetscript} with additional tile-level primitives listed in Table~\ref{tab:semantic_and_syntax}. Unlike Triton, Hexcute exposes shared memory and registers directly, enabling optimizations such as custom dataflow and pipelining. Hexcute provides control comparable to CUTLASS~\cite{Thakkar_CUTLASS_2023}, but without the heavy coding effort. CUTLASS requires users to \circleddefaults{1} select instructions, \circleddefaults{2} map threads to computational tasks, and \circleddefaults{3} specify layouts for tensors, whereas Hexcute systematically infers layouts and automates all three. 

Hexcute represents layouts as functions using CuTe’s Layout abstraction~\cite{CuTe-doc}. Shared-memory layouts are described with CuTe layouts, and register layouts are represented using thread-value layouts. CuTe distinguishes itself from other layout systems, such as 
task mappings~\cite{Hidet} and affine maps~\cite{affinemap}, by defining composition and inversion directly in the sense of function composition and inversion. Because layouts form a monoid under composition, complex layouts can be constructed systematically from simpler ones. This algebra provides Hexcute with the foundation for reasoning about and constructing layout constraints. It further allows Hexcute to reduce tensor-program optimization to solving constraints over a monoid. To our knowledge, Hexcute is the first compiler to automate CuTe layout synthesis on modern GPUs.

Fig.~\ref{fig:hexcute_overview} (b) shows a GEMM kernel in Hexcute with implicit layout specifications; its line-by-line explanation is provided in Appendix~\ref{apx:gemm}, and its formal syntax in Appendix~\ref{apx:lang_spec}. Fig.~\ref{fig:hexcute_overview} (c) then illustrates the compilation workflow. Hexcute accepts a Hexcute program and generates Hidet IR, which is lowered to CUDA code via the compiler passes in Hidet~\cite{Hidet}. 
\begin{table}[t]
{
\footnotesize
  \centering
  \caption{Hexcute tile-level primitive semantics}
  \begin{tabular}[t]{p{0.38\linewidth}p{0.52\linewidth}}
\thickhline
  \textbf{Primitives} & \textbf{Semantics} \\
\hline
  \code{global\_view(b,l)}
  & View a global memory buffer \texttt{b} as a tensor with layout \texttt{l} \\
\hline
  \code{register\_tensor(t,s)} & Create a tensor with data type \texttt{t} and shape \texttt{s} in the register files  \\
\hline
  \code{shared\_tensor(t,s)} & Create a tensor with data type \texttt{t} and shape \texttt{s} in shared memory \\
\hline
  \code{copy(a,b)} & Copy data from tensor \texttt{a} to tensor \texttt{b}  \\
\hline
  \code{gemm(c,a,b)} & Perform matmul on tensors \texttt{a}, \texttt{b}, and \texttt{c}  \\
\hline
  \code{cast(a,t)} & Cast the tensor \texttt{a} to another data type \texttt{t}  \\
\hline
  \code{rearrange(a,l)} & Redistribute register tensor across the threads to layout \texttt{l} via shared memory   \\
\hline
  \code{elementwise(a1,...,an)} & Perform elementwise operations on tensors \texttt{a1,...,an} \\
\hline
  \code{reduce(a,d)} & Reduce the tensor \texttt{a} along the dimension \texttt{d} \\
\thickhline
  \end{tabular}
  \label{tab:semantic_and_syntax}
}
\end{table}
\subsubsection*{Constraint-based Thread-Value Layout Synthesis}
During compilation, Hexcute lowers tile-level operations in Table~\ref{tab:semantic_and_syntax} to collective instructions. Following CuTe, Hexcute also models the semantics of collective instructions such as \lstinline{ldmatrix}~\cite{ldmatrix-ptx-inst} and \lstinline{mma}~\cite{mma-ptx-inst} using thread-value layouts. In Hexcute, the tile-level operations operate on tensors, whereas collective instructions operate on smaller tiles within operation-level tensors. An operation may invoke an instruction multiple times to process the entire tensor. The mapping between instruction-level tiles and tensor-level regions is not fixed by prior assumptions, and Hexcute infers it during compilation. Hexcute relates these two levels by constructing constraints through function composition and inversion, as detailed in Section~\ref{sec:tv_layout_constraint}. This provides a principled mechanism for layout propagation and systematic instruction selection. 

Hexcute builds constraints for thread-value layouts based on the semantics of collective instructions and propagates them to infer register tensor layouts, as detailed in Section~\ref{sec:tv_layout_synthesis}. For example, consider a \texttt{copy} operation that moves data from a shared memory tensor \texttt{a} to a register tensor \texttt{b}. Let $layout\_a$ denote the tensor pointer distribution and $layout\_b$ denote the register distribution. The instruction-specific constraint can be expressed as $F_I(layout\_a, layout\_b) = 0$, where $I$ denotes the instruction implementing the \lstinline{copy} operation. This formulation allows deriving forward and backward inference rules: if $layout\_a$ is known, we can derive $layout\_b$, and vice versa, as shown in \circleddefaults{1} of Fig.~\ref{fig:hexcute_overview} (c). After building these constraints, Hexcute uses them to propagate layouts starting from performance-critical operations. For example, in Fig.~\ref{fig:hexcute_overview} (c), \circleddefaults{2} propagates the layouts from the output to the inputs in a \texttt{gemm} operation. When multiple valid instruction choices exist, Hexcute enumerates them via a search tree to generate candidate programs. 

\subsubsection*{Constraint-based Shared Memory Layout Synthesis}
After synthesizing thread-value layouts, Hexcute derives shared-memory layouts for each candidate program in the search tree using a constraint-based approach. For example, in the GEMM kernel in Fig.~\ref{fig:hexcute_overview} (b), Lines 12 and 14 write to and read from the same shared memory tensor \texttt{sc}. To synthesize the layout for \texttt{sc}, Hexcute builds alignment-aware constraints for the two \lstinline{copy} operations, as shown in \circleddefaults{3} of Fig.~\ref{fig:hexcute_overview} (c). Solving these constraints yields parameterized layouts with undetermined stride variables (e.g., $D_1,...,D_n$ and $D'_1,...,D'_n$ in \circleddefaults{3} of Fig.~\ref{fig:hexcute_overview} (c)). Since both layouts refer to the same shared memory tensor, they must be compatible. Hexcute checks the satisfiability of the resulting layouts and generates code when the constraints can be satisfied. If multiple valid candidates exist, Hexcute employs the analytical cost model described in Section~\ref{sec:cost_model} to select the one with the lowest latency, as illustrated in Fig.~\ref{fig:hexcute_overview} (c). Unlike other heuristic-based layout propagation systems, this approach systematically explores optimal layouts and instructions for GPU programs. To our knowledge, Hexcute is the first compiler framework to do so.
\section{Thread-Value Layout Synthesis}
\label{sec:task_mapping_inference}
\subsection{Thread-Value Layout Constraints}
\label{sec:tv_layout_constraint}
To synthesize thread-value layouts, we construct constraints that connect operations with the hardware instructions used to implement them. We start by illustrating the derivation of a \lstinline{copy} operation moving data from tensor \texttt{a} to tensor \texttt{b}.

\subsubsection*{\texttt{Copy} Operation Constraint} Let $f$ and $g$ denote the thread-value layouts of tensors \lstinline{a} and \lstinline{b}, respectively. These layouts are functions that map a thread ID $t$ and a value index $v$ to a logical coordinate in the tensor: $f:(t, v)\mapsto(m_T, n_T),\quad g:(t, v)\mapsto(m_T, n_T)$, where $m_T$ and $n_T$ represent the row and column coordinates within the tensors. Suppose a memory instruction $I$ (used to implement the \lstinline{copy} operation) is specified by its own input and output layouts, denoted by $p$ and $q$, respectively. These layouts map the same thread-value pair to coordinates in the instruction tile. For instance, we might have: $p:(t, v)\mapsto(m_I, n_I),\quad q:(t, v)\mapsto(m_I, n_I)$, where $m_I$ and $n_I$ denote the row and column coordinates within the instruction tile. 

By computing the inverse functions $p^{-1}$ and $q^{-1}$, we can convert instruction tile coordinates back to a thread-value pair. Composing these inverses with the tensor layouts gives us the composite mappings: $f\circ p^{-1}$ and $g\circ q^{-1}$, which map coordinates from the instruction tile to the corresponding positions in the input and output tensors, respectively. Since the \lstinline{copy} operation transfers an identical logical region from the input tensor to the output tensor, these composite mappings must be identical: \(f\circ p^{-1} = g\circ q^{-1}\). 
\begin{figure}[t]
\begin{subfigure}[b]{0.49\linewidth}
        \centering
        \includegraphics[width=0.9\textwidth]{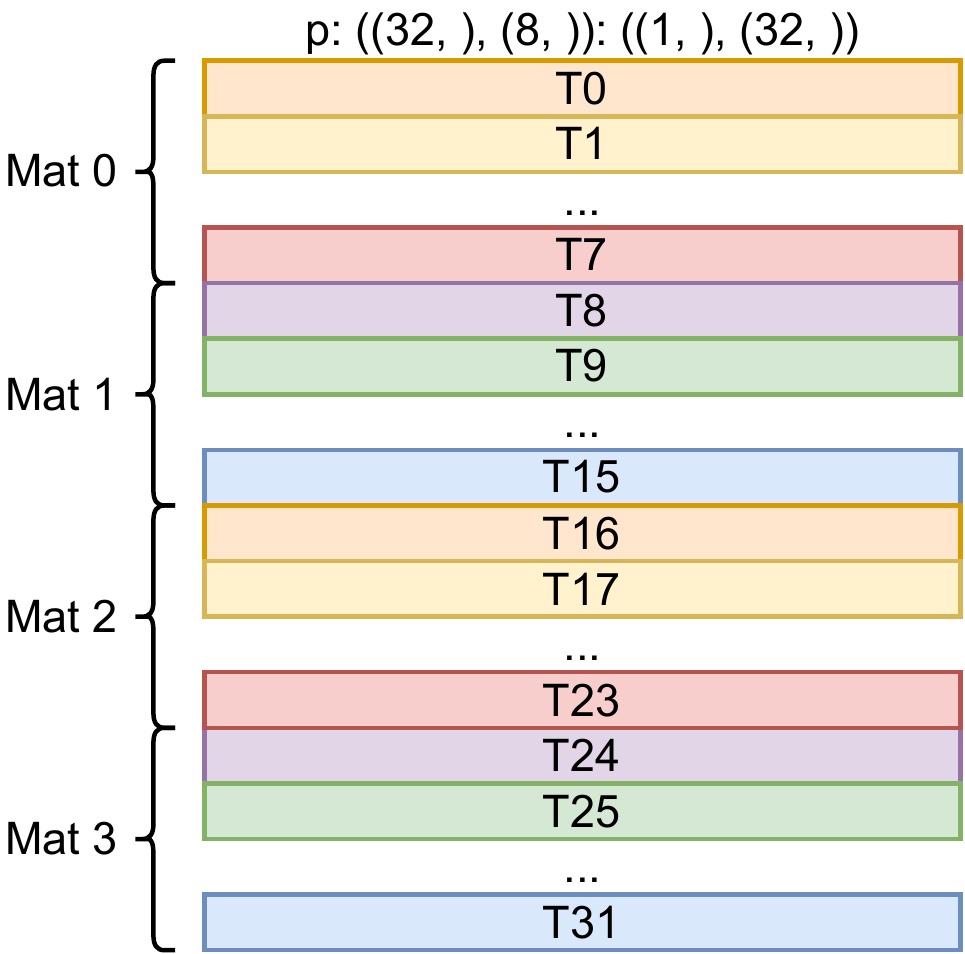}
        \caption{Layout \texttt{p}}
    \end{subfigure}
     \begin{subfigure}[b]{0.5\linewidth}
        \centering
        \includegraphics[width=0.9\textwidth]{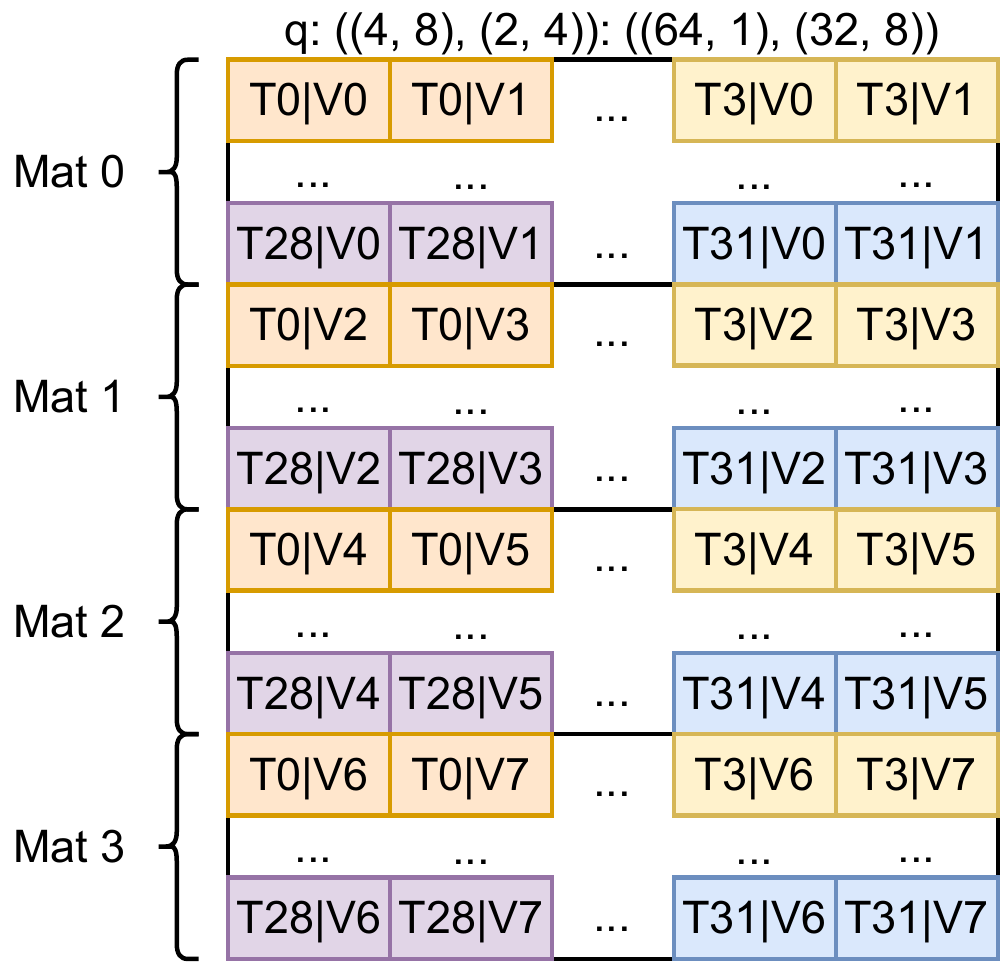}
        \caption{Layout \texttt{q}}
    \end{subfigure}
    \begin{subfigure}[t]{\linewidth}
        \centering
        \includegraphics[width=\textwidth]{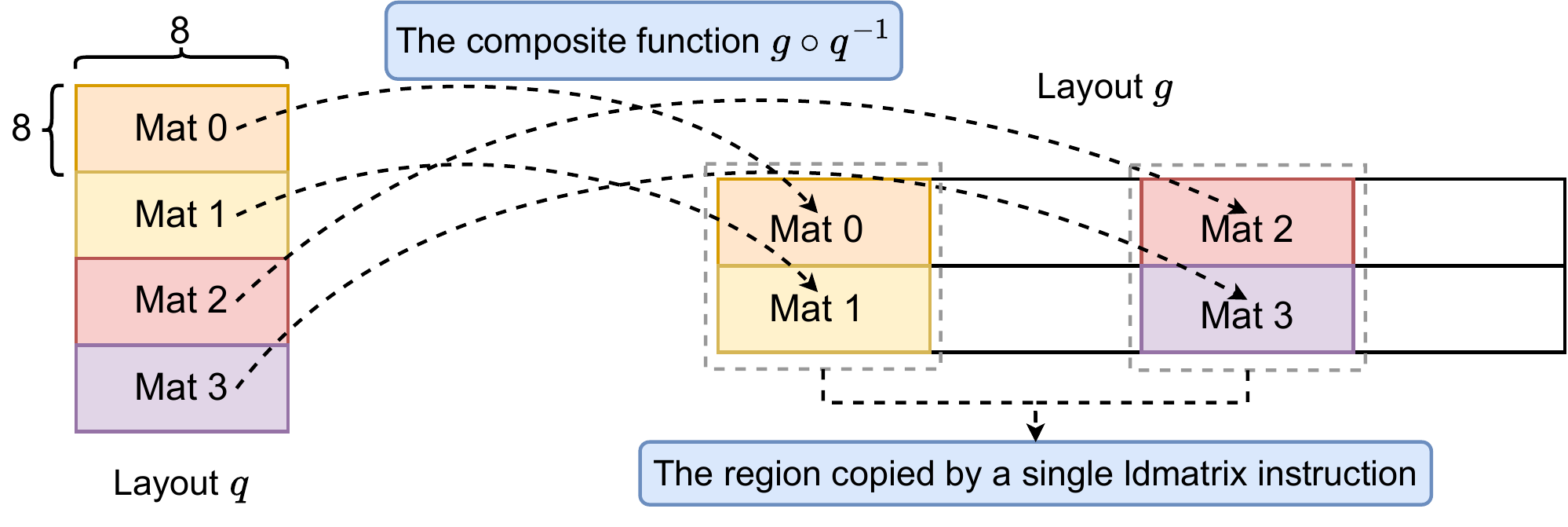}
        \caption{The composite function $g\circ q^{-1}.$}
    \end{subfigure}
\caption{The \texttt{ldmatrix} instruction is modeled by layout \texttt{p} and \texttt{q} in (a) and (b). (c) shows that the composite function $g\circ q^{-1}$ maps the four 8$\times$8 matrices loaded by the \texttt{ldmatrix} instruction to the corresponding positions in the copied tensor.}
\label{fig:mapping}
\end{figure}
\begin{figure}[t]
    \begin{subfigure}{0.50\linewidth}
        \centering
        \includegraphics[width=\textwidth]{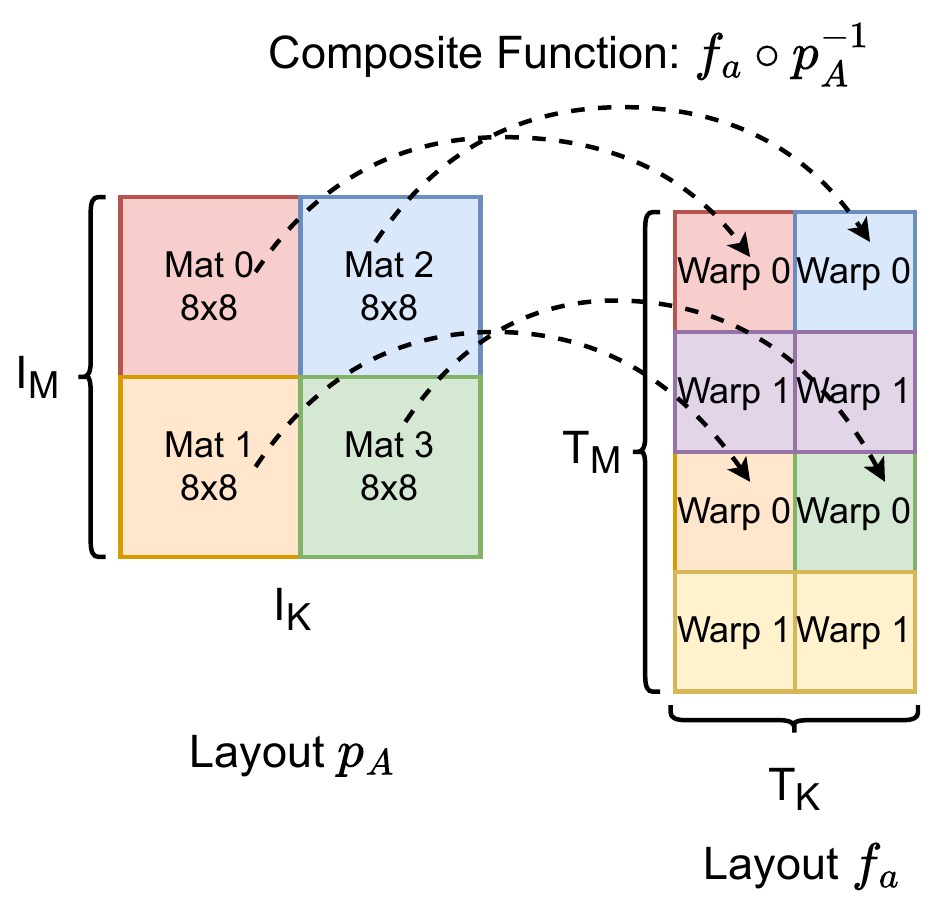}
        \caption{Composite function: \(f_a\circ p^{-1}_{A}\)}
    \end{subfigure}
    \begin{subfigure}{0.49\linewidth}
        \centering
        \includegraphics[width=\textwidth]{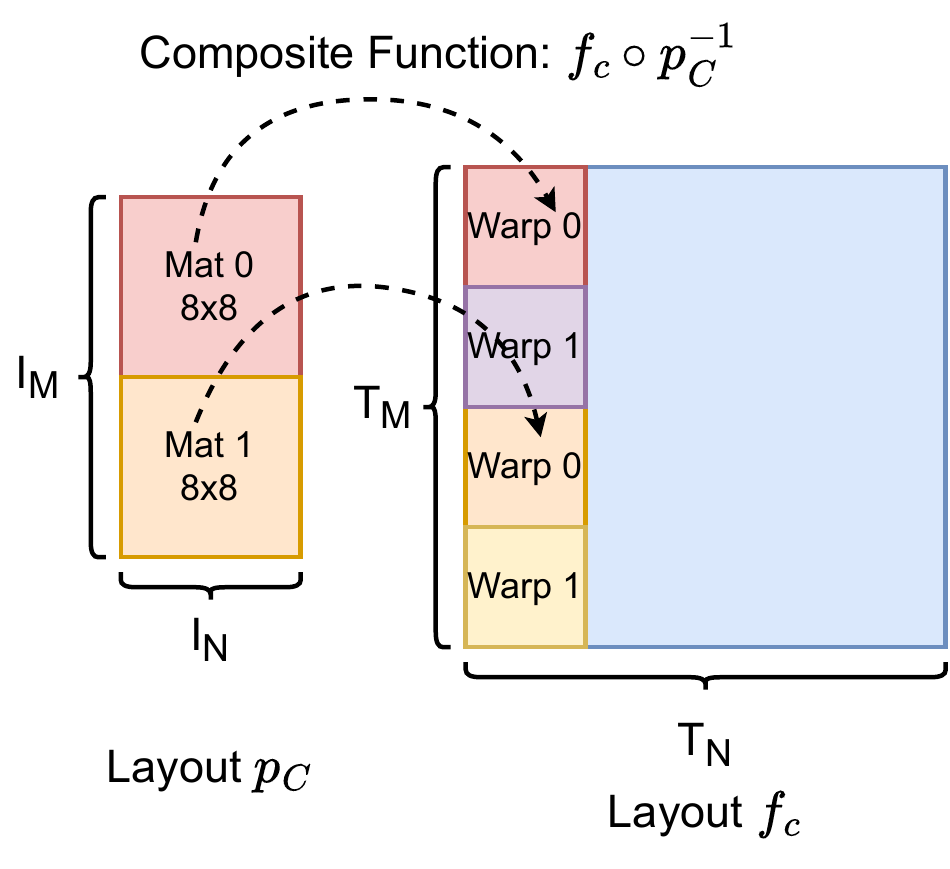}
        \caption{Composite function: \(f_c\circ p^{-1}_{C}\)}
    \end{subfigure}
  \caption{Composite mappings $f_a \circ p_A^{-1}$ and $f_c \circ p_C^{-1}$ that connect the
\texttt{gemm} operation and the \texttt{mma} instruction, showing how instruction fragments map into tensors
\texttt{a} and \texttt{c}.}
\vspace{-3mm}
\label{fig:mapping_mma}
\end{figure}

To make this concrete, consider a \texttt{copy} operation implemented with the \texttt{ldmatrix} instruction. Fig.~\ref{fig:mapping} (a) and (b) show the layouts of the \texttt{ldmatrix} instruction: 32 threads cooperatively load four 8$\times$8 matrices from shared memory, described by layout $p$, where each row is loaded by one thread. After loading, each thread holds an array of eight elements, and the register distribution is described by layout $q$. Fig.~\ref{fig:mapping} (c) then illustrates how the composite function $g \circ q^{-1}$ maps these four 8$\times$8 matrices from the instruction tile onto the corresponding regions of tensor $b$. The gray boxes highlight the region copied by a single \lstinline{ldmatrix} instruction, and the dashed arrows show how the matrices are mapped into $b$. A detailed construction of $g \circ q^{-1}$ is provided in Appendix~\ref{apx:composite_func}. To ensure consistency, the function $f \circ p^{-1}$ must define the same mapping.

\subsubsection*{\texttt{Gemm} Operation Constraint}
We apply the same principle to the \lstinline{gemm} operation. 
Let the tile-level tensors \lstinline{a}, \lstinline{b}, and \lstinline{c} have thread-value layouts 
$f_a$, $f_b$, and $f_c$. Suppose the Tensor Core instruction 
\lstinline{mma.sync.aligned.m16n8k16.row.col}~\cite{mma-ptx-inst} 
is used to implement the \texttt{gemm} operation and is specified by operand layouts 
$p_A$, $p_B$, and $p_C$. 
Fig.~\ref{fig:mapping_mma} (a) and (b) illustrate how the composite functions 
$f_a \circ p_A^{-1}$ and $f_c \circ p_C^{-1}$ map instruction-level tiles into the corresponding 
tile-level tensors.

Tensors \lstinline{a} and \lstinline{c} reside in different coordinate spaces: 
\lstinline{a} is indexed over an $M\times K$ grid, whereas \lstinline{c} uses an $M\times N$ grid. 
To relate their composite mappings, we introduce simple embedding and projection functions. 
For example, the embedding $\eta_M : m_I \mapsto (m_I, 0)$ maps the row coordinate of the 
instruction tile into the two-dimensional space of \texttt{C}, and the projection 
$\mu_M : (m_T, n_T) \mapsto m_T$ extracts the $M$-dimension coordinate from \lstinline{c}. 
Using these helper functions, we compare the restrictions of 
$f_c \circ p_C^{-1}$ and $f_a \circ p_A^{-1}$ along the $M$ dimension; both must define 
the same mapping from $I_M$ to $T_M$. 
Applying the same construction to the other tensor pairs and remaining dimensions yields the full set of constraints for the \texttt{gemm} operation, which we summarize in 
Fig.~\ref{fig:thread_value_layout_constraints} (b) in Appendix~\ref{apx:general_constraints}.

The observations above for the \texttt{copy} and \texttt{gemm} operations form the basis of 
the thread-value layout constraints. We extend and formalize these constraints for all tile-level 
operations in Fig.~\ref{fig:thread_value_layout_constraints}. 
In the next section, we present an algorithm that synthesizes these thread-value layouts 
in GPU programs.

\subsection{Thread-Value Layout Synthesis}
\label{sec:tv_layout_synthesis}
{
\begin{algorithm}[!t]
\caption{Thread-Value Layout Synthesis Algorithm}
\label{alg:task_mapping_inference}
\begin{algorithmic}[1] 
\small
    \Require A directed acyclic graph $G=(V, E)$ of tile-level operations. $V$ denotes the operators, and $E$ represents the tensors. 
    \Ensure Thread-value layouts ${L_1, L_2, \ldots, L_n}$ for all tensors in $E$  
    
    \State Partition the graph $G$ into connected components $\mathscr{S}={S_1, S_2, \ldots, S_n}$
    \For{$S_i$ in $\mathscr{S}$}
        \State $\mathscr{C}$ $\gets$ BuildConstraints($S_i$) \Comment{The set of the constraints in $S_i$}
        \State $\mathscr{L}$ $\gets \{\}$ \Comment{The set of synthesized layouts in $S_i$}
        \State $Rq$ $\gets \{\}$ \Comment{The constraints in $S_i$ that are ready to solve}
        \If{$S_i$ includes a \lstinline[basicstyle=\ttfamily\footnotesize]|gemm| operation}
            \For{\lstinline[basicstyle=\ttfamily\footnotesize]|gemm| operation in $S_i$}
                \State Select the fastest Tensor Core instruction $I$ on the target GPU.
                \State Tile matrix $C$ with instruction $I$ and instantiate layout $L_C$.
                \State Solve layouts $L_A$ and $L_B$ with constraints of \lstinline[basicstyle=\ttfamily\footnotesize]|gemm|.
                \State Update $L_A$, $L_B$, $L_C$ in $\mathscr{L}$.
            \EndFor
        \Else
            \State Pick an anchor \lstinline[basicstyle=\ttfamily\footnotesize]|copy| operation $Ac$.
            \State Instantiate layout $L$ by coalescing memory accesses.
            \State Update $L$ in $\mathscr{L}$.
        \EndIf
        \State UpdateReadyQueue($Rq$, $\mathscr{C}$, $\mathscr{L}$)
        \While{ $\mathscr{C}$ is not empty }
            \While {$Rq$ is not empty}
                \State $C$ $\gets$ Dequeue($Rq$)
                \State $L$ $\gets$ Solve($C$, $\mathscr{L}$)
                \State Update $L$ in $\mathscr{L}$.
                \State Remove $C$ from $\mathscr{C}$.
            \EndWhile
            \State UpdateReadyQueue($Rq$, $\mathscr{C}$, $\mathscr{L}$)
        \EndWhile
    \EndFor
    \State \Return
\end{algorithmic}
\end{algorithm}
}
Building on the thread-value layout constraints in Section~\ref{sec:tv_layout_constraint}, we design Algorithm~\ref{alg:task_mapping_inference} to synthesize these layouts. 
\subsubsection*{Graph Partition} 
The algorithm first constructs a directed acyclic graph (DAG) of tile-level operations and partitions it into connected subgraphs separated by shared memory reads and writes (\textbf{Line} 1). For each subgraph, the algorithm selects anchor operators and initializes the thread-value layouts for the inputs and outputs. It then builds thread-value layout constraints and propagates the layouts based on the constraints detailed in Fig.~\ref{fig:thread_value_layout_constraints}. In subgraphs containing \lstinline{gemm} operations, the \lstinline{gemm} operations are chosen as anchors due to their critical impact on kernel performance. If \lstinline{gemm} operations are absent, the algorithm selects the \lstinline{copy} operation that transfers the most data as the anchor. We choose this design because non-\lstinline{gemm} operations are typically memory-bound, which makes the \lstinline{copy} operation the most critical to kernel performance. Additionally, this is feasible because every connected component contains at least one \lstinline{copy} operation that reads or writes data; otherwise, the subgraph would be removed during dead code elimination.

\subsubsection*{Initialization} For a \lstinline{gemm} anchor, the algorithm selects the fastest Tensor Core instruction available on the target GPU and tiles the matrix \lstinline{C} with the instruction, which determines the layout for \lstinline{C} (\textbf{Lines} 8-9). It then solves the layouts for \lstinline{A} and \lstinline{B} using the constraints in Fig.~\ref{fig:thread_value_layout_constraints}~(b) (\textbf{Lines} 10-11). For a \lstinline{copy} anchor, the layout is constructed by coalescing memory accesses (\textbf{Line} 15). The process begins by sorting the memory layout dimensions by their strides. The vector size for the \lstinline{ld}/\lstinline{st} instruction is then determined by analyzing the divisibility of the strides. Finally, the thread-value layout is constructed such that consecutive threads access contiguous vectors in memory, thereby coalescing memory accesses.

\subsubsection*{Layout Solving} The algorithm maintains a set of remaining constraints, \(\mathscr{C}\), and a ready list of constraints, \(Rq\) for each subgraph (\textbf{Lines} 3-5). A constraint is added to \(Rq\) when only one thread-value layout variable remains unknown. At that point, we rewrite the constraint so that the unknown variable appears on the left-hand side of the equation (\textbf{Line} 22). For example, if the constraint for a \lstinline{copy} operation is ready and the layout \(g\) is known, we select an instruction \(I\) with input and output layouts \(p\) and \(q\), and rewrite the equation from Fig.~\ref{fig:thread_value_layout_constraints}~(a) as \(f = g \circ q^{-1} \circ p\). Using this equation, we derive the unknown layout \(f\) by applying the algebraic operations of layouts. As new layouts become known, we update the ready list accordingly. This process repeats until all constraints in \(\mathscr{C}\) are resolved.

\subsubsection*{Conflict Handling}
\begin{figure}[t]
    \begin{subfigure}[t]{\linewidth}
        \centering
        \includegraphics[width=0.8\textwidth]{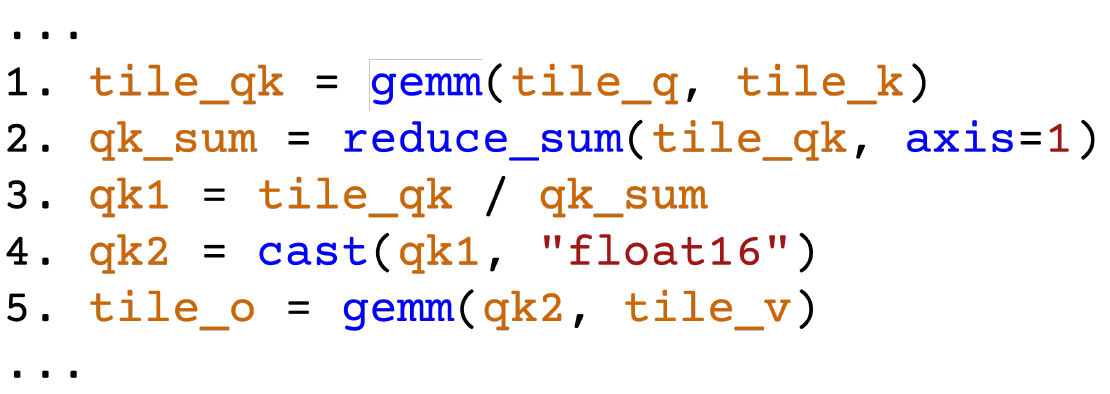}
    \end{subfigure}
\caption{Example: two \lstinline{gemm}s with a \lstinline{reduce_sum} in between.}
\vspace{-3mm}
\label{fig:flashattn}
\end{figure}
For kernels with multiple \lstinline{GEMM} operations (e.g., FlashAttention~\cite{dao2022flashattention}), the algorithm propagates thread-value layouts from all \lstinline{gemm} operations and inserts \lstinline{rearrange} operators to resolve layout conflicts. Consider an example with two \lstinline{gemm} operations and a reduction in between (Fig.~\ref{fig:flashattn}). The compiler first infers the thread-value layouts for the \lstinline{gemm} operations at \textbf{Line}~1 and \textbf{Line}~3. The constraints for the \lstinline{reduce_sum} and \lstinline{cast} operations at \textbf{Line}~2 and \textbf{Line}~4 then become solvable. Once these constraints are resolved, and the layouts of \lstinline{qk_sum}, \lstinline{tile_qk}, and \lstinline{qk1} are determined, the compiler checks for consistency to ensure that register tensors are distributed across threads in the same way. If the condition is not satisfied, the compiler inserts \lstinline{rearrange} operators to convert layouts, which requires inter-thread data exchange and may introduce extra overhead. To avoid this, Hexcute allows users to annotate a consistent thread arrangement for both \lstinline{gemm} operations, eliminating the need for layout conversion.
\subsubsection*{Expanding Search Tree}
\label{sec:expand_search_tree}
In cases where multiple valid instructions exist for a single \lstinline{copy} operation, propagating the layout for the \lstinline{copy} may lead to multiple candidate solutions. We extend Algorithm~\ref{alg:task_mapping_inference} with depth-first search (DFS) and backtracking, transforming the search space into a tree where each leaf node represents a candidate program. We build an analytical latency model, detailed in Section~\ref{sec:cost_model}, which uses the instruction microbenchmark~\cite{microbench} to estimate the latency of each leaf program. For each candidate, we synthesize shared memory layouts as described in Section~\ref{sec:shared_layout_synthesis}, discard invalid ones, and select the valid program with the lowest latency.
\begin{figure}[t]
    \begin{subfigure}[t]{\linewidth}
        \centering
        \includegraphics[width=\textwidth]{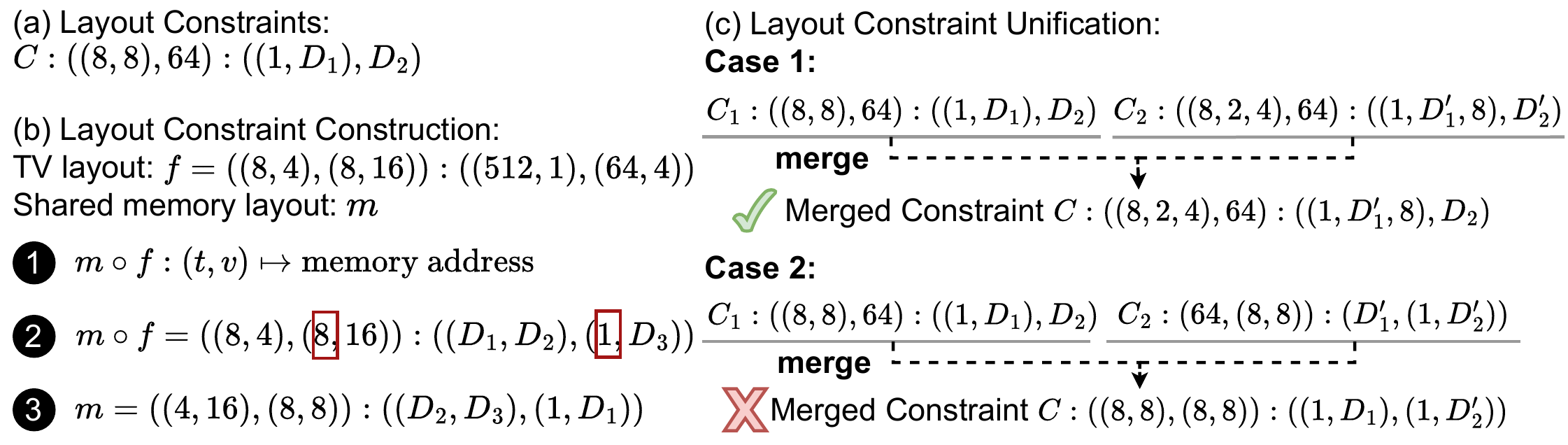}
    \end{subfigure}
\caption{Shared memory layout synthesis. (a) Shared memory layout constraints. (b) Layout constraint construction. (c) Layout constraint unification. In Case 1, the unification is successful, whereas in Case 2, it fails due to stride conflicts.}
\vspace{-3mm}
\label{fig:swizzle_functor}
\end{figure}
\section{Shared Memory Layout Synthesis}
\label{sec:shared_layout_synthesis}
For shared-memory layout synthesis, we represent the layout \(M\) as the composition
\(M = S \circ m\)~\cite{CuTe-doc}, where \(m : \mathbb{Z} \rightarrow \mathbb{Z}\) is a memory-layout
function subject to instruction alignment constraints, and \(S : \mathbb{Z} \rightarrow \mathbb{Z}\)
is a swizzle function that permutes addresses to avoid bank conflicts. Because \(m\) determines the
base memory layout and \(S\) only permutes it, Hexcute first builds layout constraints for \(m\),
derives a satisfiable layout via a layout operation called \emph{unification}, and then selects a
swizzle function \(S\) to eliminate bank conflicts.

\subsubsection*{Layout Constraint}
A layout constraint is a layout in which some modes are known while others remain undetermined. For example, consider a shared memory tensor of shape $(64,64)$. A layout constraint for this tensor could be written as $((8,8),64):((1,D_1),D_2)$, as shown in Fig.~\ref{fig:swizzle_functor} (a). Here, the modes $8:D_1$ and $64:D_2$ denote layouts of sizes 8 and 64, respectively, where $D_1$ and $D_2$ are tuples representing the strides.

\subsubsection*{Layout Constraint Construction} 
We construct layout constraints with function composition and inversion. As a concrete example, consider a \lstinline{copy} operation performed by 32 threads that loads a \lstinline{float16} tile of size 64$\times$64 from shared memory to registers. Assume a 16-byte aligned load instruction is selected, and the thread-value layout $f$ in Fig.~\ref{fig:swizzle_functor} (b), which describes how data is distributed across threads, is inferred during thread-value layout synthesis. In this case, each thread holds an 8$\times$16 local array.

The composite function $m\circ f$ defines the mapping from each thread to the memory addresses it copies, as shown in \circleddefaults{1}. Since alignment requires addresses to be multiples of 16 bytes (8 \lstinline{float16} elements), we obtain the equation in \circleddefaults{2}, where the right-hand side is a parameterized layout. The mode $8:1$ highlighted in red boxes ensures that each thread accesses 8 contiguous elements. By multiplying \circleddefaults{2} by the inverse $f^{-1}$, we derive the layout constraint shown in \circleddefaults{3} of Fig.~\ref{fig:swizzle_functor} (b).

\subsubsection*{Layout Constraint Unification}
The unification of multiple layout constraints aims to find a layout constraint that incorporates all determined modes from the inputs without conflicts. Fig.~\ref{fig:swizzle_functor} (c) illustrates two cases. \textbf{Case 1} shows the unification of two layout constraints $C_1$ and $C_2$. The resulting layout is \(C = ((8,2,4),64):((1,D_1',8),D_2)\), which contains all determined modes in both $C_1$ and $C_2$. During unification, the mode $8:D_1$ from $C_1$ is split into two modes, yielding $(2,4):(D_1',8)$, a refinement of the original mode. In \textbf{Case 2}, no valid layout can be found. Merging $C_1$ and $C_2$ produces \(((8,8),(8,8)):((1,D_1),(1,D_2'))\), which is invalid because it includes two modes of shape 8 and stride 1. This causes different data elements to map to the same memory address, violating the requirement that a shared memory layout must represent a contiguous memory space.


\subsubsection*{Integration}
Based on the analysis above, the compiler traverses the program and applies the approach to each \lstinline{copy} operation involving shared memory tensors. It merges layout constraints from all \lstinline{copy} operations to construct a unified layout constraint that satisfies all alignment requirements. If the constraints are satisfiable, the compiler materializes the undetermined strides and synthesizes a concrete layout $m$.

If the layout constraints cannot be unified, the compiler falls back to scalar instructions. As described in Section~\ref{sec:expand_search_tree}, the DFS-based approach ensures that one valid fallback program exists, in which all \lstinline{copy} operations use scalar instructions.

For \lstinline{copy} operations that use Tensor Memory Accelerator (\lstinline{TMA}) instructions, the thread-value layout approach is not applicable because \lstinline{TMA} is issued by a single thread. We therefore first merge and resolve all layout constraints from non-\lstinline{TMA} operations. We then check whether the resulting shapes and strides satisfy the requirements of \lstinline{TMA}. If so, we instantiate the unknown strides with a heuristic that maximizes the tensor region copied by each \lstinline{TMA} instruction while minimizing the number of issued instructions. If the constraints for \lstinline{TMA} cannot be satisfied, the compiler falls back to 
non-\lstinline{TMA} instructions. 

\subsubsection*{Eliminating Bank Conflicts} 
To eliminate bank conflicts in shared memory accesses, we leverage the generic swizzle function defined in CuTe~\cite{CuTe-doc}. We build a compiler pass that traverses the program and identifies all \lstinline{copy} operations that access the same shared memory buffer. The pass then enumerates the applicable swizzle functions and selects the one that most effectively reduces bank conflicts. This is done by analyzing the memory addresses accessed by the threads within a warp and choosing the swizzle function that minimizes overlapping accesses to the same bank.

\section{Analytical Cost Model}
\label{sec:cost_model}
When multiple memory instructions are available for a copy operation, Hexcute expands layout propagation into a search tree, where each leaf node corresponds to a candidate program with a different combination of instructions. To choose the best candidate without exhaustively compiling and running all of them, Hexcute employs an analytical cost model inspired by the instruction scheduling algorithm~\cite{instsche}.

We model a candidate program as a sequence of tile-level operations $O_1, O_2, \ldots, O_n$. Modern GPUs allow memory operations to run in flight and overlap with computation, so the cost model tracks both the issue cycles and the completion cycles of each operation. Let $L_k$ denote the cumulative cycles required to issue the first $k$ operations, and let $S_k$ be the set of operations still in flight when issuing $O_k$. The transition from $L_k$ to $L_{k+1}$ depends on the issue latency of $O_{k+1}$ and the additional delay caused by its read-after-write (RAW) dependencies. Suppose an instruction $I$ is chosen to implement $O_{k+1}$. The issue latency of $O_{k+1}$ is computed by multiplying the number of invocations of $I$ by the issue latency of a single instance of $I$. The number of invocations is determined by the operand layouts of $O_{k+1}$. We use microbenchmarks~\cite{microbench} to obtain the instruction $I$'s issue and completion cycles, which together determine both the total issue latency and the remaining completion latency of $O_{k+1}$. These values update the set $S_{k+1}$ to reflect operations still in flight.

We repeat this process for all operations in the sequence and add the remaining latencies of unfinished operations to obtain the estimated execution latency of the candidate program.
\section{Evaluation}
We implemented Hexcute in Python with about 10K lines of code. We extend the Hidet Script~\cite{hidetscript} with tile-level primitives and lower these primitives into Hidet~\cite{Hidet} IR by analyzing the synthesized layouts. Hexcute also supports warp specialization~\cite{warpspecialization}, a technique for
overlapping memory accesses with computation, through NVDSL-style~\cite{nvdsl} Python context managers such as \texttt{warp\_groups\_producer} and \texttt{warp\_groups\_consumer}. These constructs serve as syntactic sugar for writing warp-specialized kernels, and our evaluation shows that layout synthesis is fully compatible with this optimization. We then evaluate Hexcute on NVIDIA A100 and H100 GPUs via extensive benchmarks. In these benchmarks, we set the GPU frequency to \SI{1.41}{GHz} to ensure stable results.

\begin{table*}[t!]
\caption{Programmability and performance comparison of Hexcute, CUDA, and Triton on various GPU operations. Speedups (higher is better) are the geometric means over manually optimized CUDA baselines, evaluated across multiple input shapes. CUDA kernels are implemented using different frameworks. Hexcute reduces the amount of code by 1.27$\times$-7.94$\times$ compared to CUDA by automating tensor layout synthesis. While Hexcute kernels are sometimes more verbose than Triton, Hexcute supports explicitly expressing fine-grained optimizations, achieving 1.13$\times$-2.34$\times$ speedups over Triton.}
\begin{center}
\begin{tabular}{cccccccccc}
\hline
 \multirow{2}{*}{\textbf{GPU}} & \multirow{2}{*}{\textbf{Operator}} & \multirowcell{2}{\textbf{Evaluated}\\\textbf{Shapes}} & \multicolumn{3}{c}{\textbf{Lines of Code}}&\multirow{2}{*}{\textbf{Performance Baseline}}&\multicolumn{2}{c}{\textbf{Normalized Performance}} \\
 & & & \textbf{CUDA}& \textbf{Triton}& \textbf{Hexcute}& &\textbf{Triton}& \textbf{Hexcute}
\\
\hline
\multirowcell{3}{NVIDIA\\ A100 GPU} & FP16 GEMM&40& 703$^\mathrm{b}$ & 71 & 98 & cuBLAS & 0.75$\times$ & \textbf{1.00$\times$} &  \\
& Fused MHA$^\mathrm{a}$ Forward&20& 577 & 114 & 172 & FlashAttention2 & 0.93$\times$ & \textbf{1.05}$\times$ \\
& Fused MHA Decoding&24& 322 & 224 & 253 & FlashInfer & 0.50$\times$ & \textbf{1.02}$\times$ \\
\hline
\multirowcell{3}{NVIDIA \\ H100 GPU} & Blockwise Scaled FP8 GEMM & 35 & 900 & 87 & 180 & CUTLASS & 0.50$\times$ & \textbf{1.17}$\times$\\
& Warp Specialized FP16 GEMM & 40 & 1024$^\mathrm{b}$ & 71 & 169 & cuBLAS & 0.64$\times$ & \textbf{1.25}$\times$ \\
& Fused MHA Forward & 20 & 1684 & 114 & 212 & FlashAttention3 & 0.56$\times$ & \textbf{1.27$\times$} \\
\hline
\multicolumn{10}{l}{$^{\mathrm{a}}$ Multi-head Attention}\\
\multicolumn{10}{l}{$^{\mathrm{b}}$ For GEMM, speedups are reported against cuBLAS, while LoC comparisons use CUTLASS because cuBLAS is closed‑source.}
\end{tabular}
\vspace{-3mm}
\label{general-operators}
\end{center}
\end{table*}

\subsection{Kernel Programmability and Performance}
\label{evl:programmability}
We compare the programmability and performance of Hexcute with best-effort CUDA kernels and Triton on a wide range of deep learning operators, using A100 and H100 GPUs. Table~\ref{general-operators} reports the lines of code and the performance of these operators. Performance is normalized against CUDA baselines and summarized as geometric means across various input shapes. The CUDA baselines include cuBLAS~\cite{cublas} (v12.4), CUTLASS (v4.1.0), FlashAttention~\cite{dao2022flashattention} (v2.8.2), and FlashInfer~\cite{flashinfer} (v0.2.10), while Triton is evaluated with v3.3.0. 

We summarize our findings as follows. Hexcute reduces the amount of code by 1.27$\times$-7.94$\times$ compared to CUTLASS and CUDA kernels by automating layout synthesis and removing manual specifications. For example, FlashAttention3~\cite{flashattention3} implements the attention layer using CUTLASS in about 1600 lines of code (LoC), with roughly 300 LoC dedicated to the template specifications for layouts and the rest for the kernel body, whereas Hexcute implements the same kernel in only 212 LoC. In Hexcute, users still manually express the algorithm logic in the kernel body, but the compiler automatically infers tensor layouts that are manually specified in FlashAttention3.

Hexcute generates shape-specific kernels and supports dynamic shapes via symbolic dimensions. Across a range of shapes, Hexcute matches or surpasses manually optimized CUDA libraries, demonstrating the robustness and generality of our layout synthesis algorithm. In the warp-specialized GEMM benchmark, Hexcute outperforms cuBLAS, the strongest GEMM baseline, by exhaustively tuning kernel hyperparameters, such as tile sizes. We find that non-power-of-two tiles are critical, as 28 of 40 shapes select them, and disabling these choices degrades performance by up to 13.4\% (from \SI{500}{TFLOPS} to \SI{433}{TFLOPS}) for some shapes on H100.

Compared to Triton, the LoC in Hexcute kernels ranges from 1.12$\times$ to 2.38$\times$ that of Triton kernels. Hexcute kernels are sometimes more verbose because they explicitly express critical optimizations such as optimal dataflow, warp specialization, and fine-grained software pipelining. Our layout synthesis algorithm is fully compatible with these techniques. In contrast, Triton implements these optimizations in the compiler as hard-coded heuristics, limiting its programming flexibility. By fine-tuning these optimization strategies, Hexcute achieves a 1.95$\times$-2.34$\times$ speedup over Triton in benchmarks on the H100 GPU.

\subsection{Complex Deep Learning Layers}
To demonstrate that Hexcute's layout synthesis algorithm improves performance via systematic instruction selection, we evaluate two complex layers used in deep learning workloads: mixed-type MoE~\cite{MoE} and Mamba~\cite{mamba} scan. We compare Hexcute against Triton and hand-tuned kernel libraries. Specifically, for the mixed-type MoE, we compare against two Marlin-style kernels
from two versions of vLLM (v0.8.2 and v0.9.2), and for the Mamba scan, we compare against the Mamba library~\cite{mambalib}.

\subsubsection*{Mixed-type MoE}
\label{evl:mixed_type_moe}
On an H100 GPU, Hexcute consistently outperforms Triton and the older Marlin implementation ("Marlin-old," v0.8.2) across various token numbers, achieving an average speedup of 6.46$\times$ over Triton and 28.42$\times$ over Marlin-old, as shown in Fig.~\ref{fig:moe}. Compared to the latest Marlin version ("Marlin-new," v0.9.2), Hexcute achieves comparable performance ranging from 0.89$\times$ to 1.01$\times$. Hexcute outperforms Triton because it explicitly expresses the optimal dataflow in Fig.~\ref{fig:dataflow_mixed_gemm} (b), avoiding the suboptimal dataflows caused by Triton’s heuristics. Furthermore, Hexcute’s layout inference algorithm synthesizes the optimal tensor layouts in Fig.~\ref{fig:int4_packed_layout}, leading to better instruction selection and improved performance. Table~\ref{tab:insts_moe} highlights that Hexcute extensively utilizes vectorized instructions to maximize memory bandwidth, while Triton relies on scalar instructions in some cases, limiting its performance. 

\begin{figure}[t!]
    \begin{subfigure}{\linewidth}
        \centering
        \includegraphics[width=\textwidth]{Figures/MoE_mixed_type.pdf}
    \end{subfigure}
  \caption{Latency comparison for an MoE layer with 256 experts on the H100 GPU across Marlin-old, Triton, Marlin-new, and Hexcute. Hexcute achieves an average speedup of 6.46$\times$ over Triton and 28.42$\times$ over Marlin-old, while reaching approximately 96\% of Marlin-new's performance on average.}
  \vspace{-2mm}
\label{fig:moe}
\end{figure}
\begin{table}[t!]
  \centering
  \caption{Bytes per instruction for the MoE kernels with the same tile size in Hexcute and Triton. Larger numbers indicate wider, more vectorized instructions. Hexcute consistently generates wider instructions compared to Triton, resulting in improved memory bandwidth utilization.}
    \begin{tabular}{ cccccc }
     \hline
     Tensor & \multicolumn{3}{c}{Triton (bytes)} & \multicolumn{2}{c}{Hexcute (bytes)} \\
     \hline
     \texttt{A} & 16$^{\mathrm{G2S}}$ & 8$^{\mathrm{S2R}}$ & & 16$^{\mathrm{G2S}}$ & 8$^{\mathrm{S2R}}$ \\
     \texttt{B} & 1$^{\mathrm{G2R}}$ & 2$^{\mathrm{R2S}}$ & 8$^{\mathrm{S2R}}$ & 16$^{\mathrm{G2S}}$ & 8$^{\mathrm{S2R}}$\\
     \texttt{scale} & 16$^{\mathrm{G2S}}$ & 2$^{\mathrm{S2R}}$ & & 16$^{\mathrm{G2S}}$ & 8$^{\mathrm{S2R}}$  \\
     \texttt{zp}$^{\mathrm{a}}$ & 1$^{\mathrm{G2R}}$ & & & 16$^{\mathrm{G2S}}$ & 8$^{\mathrm{S2R}}$ \\
     \hline
     \multicolumn{4}{l}{$^{\mathrm{G2S}}$ Global memory to shared memory copy.} \\
     \multicolumn{4}{l}{$^{\mathrm{G2R}}$ Global memory to register file copy.} \\
     \multicolumn{4}{l}{$^{\mathrm{S2R}}$ Shared memory to register file copy.} \\
     \multicolumn{4}{l}{$^{\mathrm{R2S}}$ Register file to shared memory copy.} \\
     \multicolumn{4}{l}{$^{\mathrm{a}}$ Zero-point tensor.} \\
    \end{tabular}
\label{tab:insts_moe}
\end{table}
\begin{table}[t!]
  \centering
  \caption{Bytes per instruction generated by Hexcute and used in the handwritten Mamba library for the selective scan kernel. Larger numbers indicate wider, more vectorized instructions. Hexcute selects wider instructions, enabling higher memory bandwidth utilization.}
    \begin{tabular}{ W{c}{1cm} W{c}{1.5cm} W{c}{1.5cm} W{c}{1.5cm} }
     \hline
      \multirow{2}{*}{Tensor} & Mamba (bytes) & \multicolumn{2}{c}{Hexcute (bytes)} \\
      & G2R$^{\mathrm{a}}$ & G2S$^{\mathrm{b}}$ & S2R$^{\mathrm{c}}$ \\
     \hline
     \texttt{u} & 2 & 16 & 16 \\
     \texttt{$\Delta$} & 2 & 16 & 16 \\
     \texttt{A} & 4 & 4 & 8 \\
     \texttt{B} & 2 & 16 & 16 \\
     \texttt{C} & 2 & 16 & 16 \\
     \texttt{Z} & 2 & 16 & 16 \\
     \hline
     \multicolumn{4}{l}{$^{\mathrm{a}}$ Global memory to register file copy.} \\
     \multicolumn{4}{l}{$^{\mathrm{b}}$ Global memory to shared memory copy.} \\
     \multicolumn{4}{l}{$^{\mathrm{c}}$ Shared memory to register file copy.} \\
    \end{tabular}
\vspace{-3mm}
\label{tab:insts_in_kernel}
\end{table}
Hexcute significantly outperforms Marlin-old due to the high kernel-launch overhead of the Marlin-old implementation, which launches a separate GEMM kernel for each expert. The Marlin-old kernel is manually adapted from Marlin's mixed-type GEMM kernels and contains 1411 LoC. While Hexcute delivers moderate performance improvements compared to Marlin-new, Marlin-new requires approximately 1889 LoC and took around two months to implement before being released in vLLM v0.8.5. In contrast, Hexcute’s mixed-type MoE kernel comprises only 342 LoC and was developed within one week of the DeepSeek-R1-AWQ model's release.

\subsubsection*{Mamba Scan}
\label{evl:mamba_scan}
The selective scan operator is central to selective state-space models (SSMs)~\cite{mamba}. We implement this kernel in Hexcute and compare it against the hand-optimized Mamba library on the H100 GPU. Across 20 different input sizes, Hexcute achieves an average speedup of 4.17$\times$ over the Mamba library. The kernel scans over the sequence dimension and performs arithmetic and reduction operations on six tensors: $u$, $\Delta$, $A$, $B$, $C$, and $Z$. Efficiently loading these tensors across GPU memory hierarchies and parallelizing computations to GPU threads poses significant challenges. Hexcute outperforms Mamba because its layout synthesis and instruction selection enable wide, vectorized memory instructions. Table~\ref{tab:insts_in_kernel} compares the load and store instructions generated by Hexcute with those used in the Mamba library. The Mamba library relies on \lstinline{cub::BlockLoad}~\cite{cub} for coalesced global memory accesses; however, it cannot select optimal instructions and instead leverages scalar loads and stores. In contrast, Hexcute generates vectorized load/store instructions that fully exploit the GPU's memory bandwidth. 

Moreover, software pipelining improves this layer's performance by up to 16\%. Hexcute makes this tuning easy by letting users specify the high-level kernel skeleton, while the compiler automatically determines layouts and selects instructions.
\begin{figure}[t!]
    \begin{subfigure}{\linewidth}
        \centering
        \includegraphics[width=\textwidth]{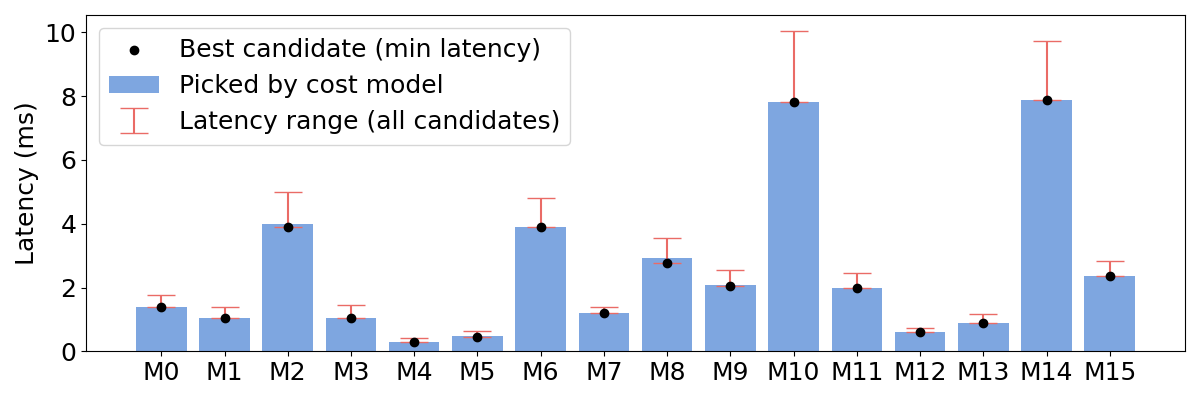}
    \end{subfigure}
  \caption{Accuracy of the analytical cost model across 16 GEMM shapes. The model consistently selects the best candidates within 1.01$\times$ of the true optimal latency.}
\label{fig:cost_model}
\end{figure}
\subsection{Compilation Time}
\label{evl:comp_time}
The compilation time of Hexcute is comparable to that of Triton. For example, compiling a GEMM layer with 102 kernel candidates takes \SI{48.39}{seconds} in Hexcute on an H100 GPU with 20 CPU cores, compared to \SI{57.10}{seconds} in Triton. Although each kernel candidate yields multiple programs with different layouts and instruction choices, Hexcute uses an analytical cost model to estimate their latencies and selects the fastest one. As shown in Fig.~\ref{fig:cost_model}, the analytical cost model selects candidates within 1.01$\times$ of the true optimal latency. 
\begin{figure}[t!]
    \begin{subfigure}{\linewidth}
        \centering
        \includegraphics[width=\textwidth]{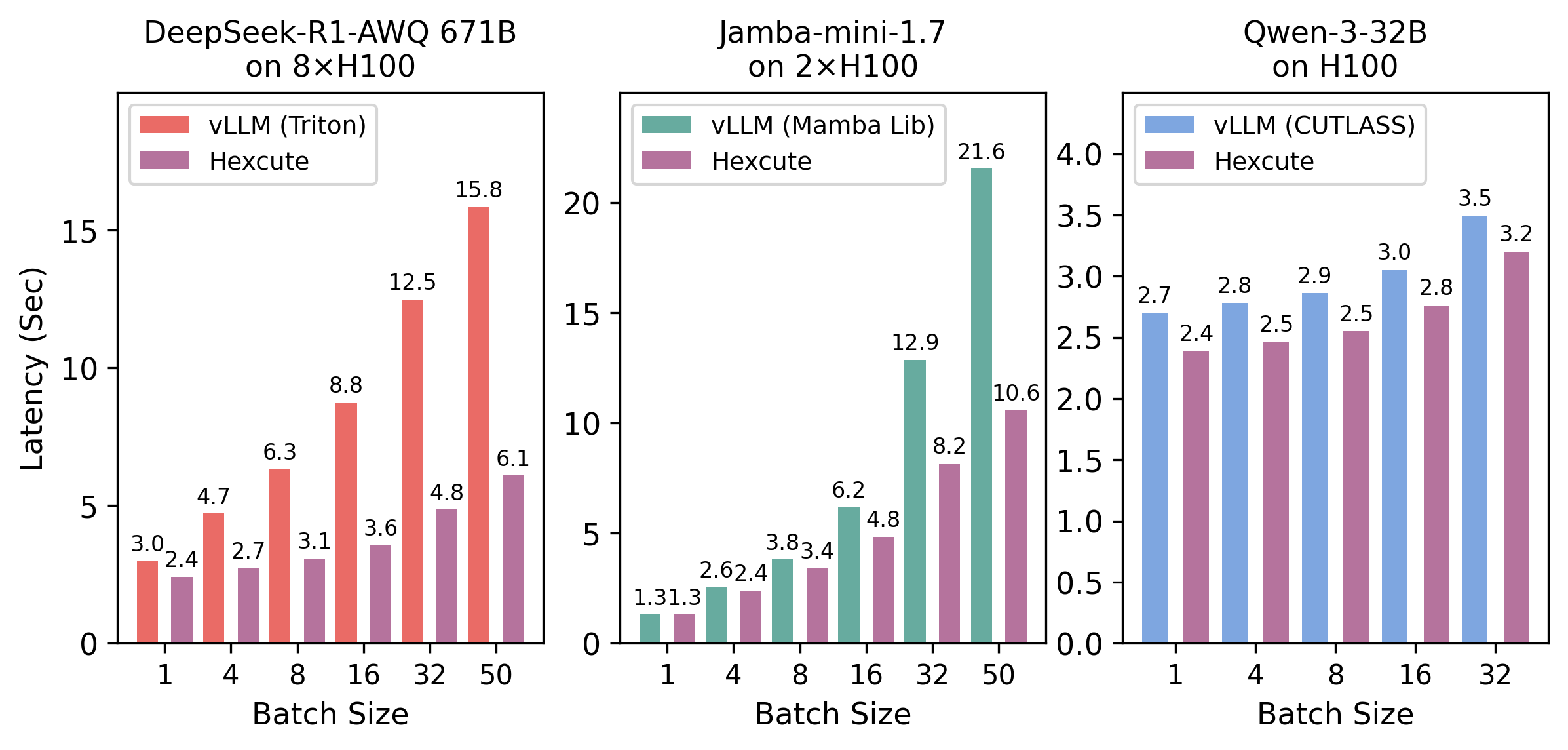}
    \end{subfigure}
  \caption{End-to-end latency of vLLM with Hexcute-integrated kernels on DeepSeek-R1-AWQ, Jamba-mini-1.7, and Qwen-3-32B. Hexcute achieves up to 2.60$\times$, 2.04$\times$, and 1.13$\times$ speedups on the three workloads.}
\vspace{-3mm}
\label{fig:e2e}
\end{figure}
\subsection{End-to-end Evaluation}
We integrate the mixed-type MoE, Mamba scan, and blockwise scaled FP8 GEMM kernels into vLLM and evaluate three end-to-end workloads: DeepSeek-R1-AWQ, Jamba-mini-1.7, and Qwen-3-32B. The latency for generating 100 output tokens is reported in Fig.~\ref{fig:e2e} across multiple H100 GPUs. On DeepSeek-R1-AWQ, Hexcute achieves an average speedup of 2.04$\times$ (up to 2.60$\times$) over the original vLLM Triton implementation. On the Mamba-based model, Hexcute achieves an average speedup of 1.30$\times$ (up to 2.04$\times$). For Qwen-3-32B, Hexcute outperforms the CUTLASS-backed vLLM implementation, achieving up to 1.13$\times$ lower latency.

\subsection{Ablation Study}
\label{evl:ablation}
We conduct an ablation study on an H100 GPU using the same MoE layer evaluated in Fig.~\ref{fig:moe}. Reproducing Triton's dataflow or enforcing Triton's shared memory layout in the Hexcute kernel results in average performance degradations of 28.5\% and 37.5\%, respectively. This demonstrates that both the efficient dataflow and the layout synthesis are critical for achieving peak performance. Fig.~\ref{fig:ablation} shows the results across varying token numbers. Notably, even when Hexcute reproduces Triton’s dataflow, it still outperforms Triton’s MoE kernel, underscoring the effectiveness of Hexcute’s layout synthesis algorithm. In this dataflow, Hexcute moves weights across memory hierarchies with 16-byte instructions, which are wider and more vectorized than those used in the Triton kernel. 
\begin{figure}[t!]
    \begin{subfigure}{\linewidth}
        \centering
        \includegraphics[width=\textwidth]{Figures/ablation_study_4.pdf}
    \end{subfigure}
  \caption{Ablation study on the MoE layer in Fig.~\ref{fig:moe}, showing average performance degradations of 28.5\% and 37.5\% when reproducing Triton's dataflow or tensor layouts in Hexcute.}
\vspace{-3mm}
\label{fig:ablation}
\end{figure}
\section{Related Work}
Hexcute is a compiler framework that automates layout synthesis in GPU programs. It models layouts as functions and employs a type-inference-based algorithm to generate complex layouts. Prior work either requires manual layout annotations, relies on layout systems that lack generality, or depends heavily on heuristics. To our knowledge, Hexcute is the first compiler framework to systematically explore layouts and instructions on modern GPUs. We summarize related work below:

\subsubsection*{Loop-Transformation-Based Compilers}
Most existing deep learning compilers, such as Halide~\cite{Halide}, TVM~\cite{Chen2018TVMAA}, FlexTensor~\cite{FlexTensor}, TensorIR~\cite{feng2022tensorir}, SparsIR~\cite{sparsetir}, Ansor~\cite{Ansor}, and AMOS~\cite{AMOS}, rely on loop transformation primitives to optimize tensor computations. Ladder~\cite{wang2024ladder}, a recent TVM-based~\cite{Chen2018TVMAA} system, extends this approach to low-precision computations. However, loop transformation primitives still require manual layout annotations and cannot automatically infer the sophisticated layouts needed for modern GPUs. In contrast, Hexcute automatically generates layout specifications in GPU programs.

\subsubsection*{Programming Frameworks with Manual Layout Specification}
Programming frameworks with explicit layout specifications generally offer greater expressiveness. Graphene~\cite{graphene} introduces a Layout abstraction to express GPU optimizations. Hidet~\cite{Hidet} employs task-mapping constructs to represent tensor layouts. ThunderKittens~\cite{thunderkittens} is implemented as a C++ template library adopting a tile-based programming model, but it requires programmers to explicitly specify the layouts for tiles. CuTe~\cite{CuTe-doc,Thakkar_CUTLASS_2023}, another C++ template library, utilizes the Layout concept to represent tensor layouts and provides algebraic operations for constructing layouts. However, all these frameworks require manual layout specification, which is error-prone and requires intimate hardware knowledge. In contrast, Hexcute automates layout specifications and thus reduces programming effort.

\subsubsection*{Programming Languages with Automatic Layout Synthesis}
Triton~\cite{Triton} is a tile-based programming language that simplifies GPU programming effort. Initially, Triton employs a case-by-case approach to modeling tensor layouts, which lacks extensibility for emerging deep learning operators. LinearLayout~\cite{linearlayouts} is a recent layout system proposed to address Triton’s limitations. However, LinearLayout does not support non-power-of-two tiles, which are essential for modern GPUs like Hopper and Blackwell that natively support non-power-of-two shaped instructions. Additionally, LinearLayout~\cite{linearlayouts} does not treat tensor layouts as general functions, limiting its ability to synthesize optimal layouts for mixed-type operations. In contrast, Hexcute leverages CuTe’s layout concept~\cite{CuTe-doc}, which inherently supports non-power-of-two tiles. It treats layouts as functions and synthesizes layouts with a type-inference-based algorithm, which systematically explores layouts and instructions on modern GPUs.
\section{Discussion, Limitations, and Future Work}



This paper presents a principled mechanism for synthesizing tensor layouts on modern GPUs, but the approach still has several limitations.

Hexcute synthesizes layouts from performance-critical operators such as \texttt{gemm}. When multiple \texttt{gemm} operations appear in a program and are connected through register tensors, the compiler may introduce register-layout conversions to maintain correctness, which incurs overhead. Hexcute currently requires users to annotate consistent thread arrangements for each \texttt{gemm}; removing this restriction is our future work, potentially by analyzing dependencies and propagating layouts only from constraint-free \texttt{gemm} operations.

Hexcute also exhaustively explores shared-memory layouts that satisfy alignment constraints. Although a naïve implementation is exponential in the number of shared-memory buffers, most buffers in real programs are independent. In such cases, applying the analytical cost model to each buffer independently avoids searching the full space and reduces the effective complexity to linear. Formalizing this decomposition and further pruning the search space remain important directions.
\section{Conclusions}
This paper presents Hexcute, a tile-based compiler framework that automates layout synthesis on GPUs. By treating layouts as functions and solving constraints with a type-inference algorithm, Hexcute systematically explores layouts and instructions while still supporting explicit optimizations such as dataflow and pipelining. It matches or exceeds expert-tuned CUDA libraries, while reducing coding effort by 1.27$\times$-7.94$\times$, outperforms Triton by 6.46$\times$ on mixed-type operators, and achieves up to 2.60$\times$ speedups in the end-to-end evaluation. 

\section*{Acknowledgement}
We sincerely thank the anonymous reviewers for their valuable feedback. We are also grateful to the members of the EcoSystem Research Laboratory at the University of Toronto for their discussions and suggestions, with special thanks to Christina Giannoula. We also thank Shang Wang and Kaihang Jiang at NVIDIA for their contributions, and Jie Ren at KAUST for his feedback on the presentation. The authors at the University of Toronto are supported by Vector Institute research grants, the Canada Foundation for Innovation JELF grant, the NSERC Discovery Grant, the Google Scholar Research Award, and the VMware Early Career Faculty Grant.

\appendices
\section{Artifact Appendix}

\subsection{Abstract}
We provide the source code for Hexcute, along with scripts to reproduce the experimental results reported in the evaluation section. This appendix contains step-by-step instructions for generating the geometric-mean speedups in Table~\ref{general-operators}, as well as reproducing the plots corresponding to Fig.~\ref{fig:moe}, Fig.~\ref{fig:cost_model}, Fig.~\ref{fig:ablation}, Fig.~\ref{fig:selective_scan}, and Fig.~\ref{fig:ampere_gemm}–Fig.~\ref{fig:hopper_mha}. We also provide a Docker image that includes the complete runtime environment.

All experiments must be run on an x86-64 Linux host equipped with at least 20 CPU cores, 100 GB of RAM, and 100 GB of free disk space. The evaluation requires access to two GPUs: an NVIDIA A100 PCIe (80 GB) and an NVIDIA H100 PCIe or SXM (80 GB). The experiments were conducted on the H100 PCIe GPU. The results on the H100 SXM GPU may vary slightly, but the overall trends should match those reported in the paper. 

\subsection{Artifact check-list (meta-information)}

{\small
\begin{itemize}
  \item {\bf Compilation:} The artifact is pre-compiled. 
  \item {\bf Run-time environment:} Docker.
  \item {\bf Hardware:} A 20-core CPU with 100 GB RAM, an NVIDIA A100 PCIe GPU (80 GB), and an NVIDIA H100 PCIe or SXM GPU (80 GB).
  \item {\bf Metrics:} Average latency, throughput, and geometric-mean speedups.
  \item {\bf Output:} The geometric-mean speedups in Table~\ref{general-operators} and plots similar to Fig.~\ref{fig:moe}, Fig.~\ref{fig:cost_model},  Fig.~\ref{fig:ablation},  Fig.~\ref{fig:selective_scan}, and Fig.~\ref{fig:ampere_gemm}–Fig.~\ref{fig:hopper_mha}.
  \item {\bf Experiments:} Kernel performance evaluation (Section~\ref{evl:programmability}), mixed-type MoE and Mamba scan evaluation (Section~\ref{evl:mixed_type_moe}), cost model accuracy evaluation (Section~\ref{evl:comp_time}), and ablation study (Section~\ref{evl:ablation}).
  \item {\bf How much disk space required (approximately)?:} Approximately 100 GB.
  \item {\bf How much time is needed to prepare workflow (approximately)?:} About 10 minutes to download and initialize the Docker image.
  \item {\bf How much time is needed to complete experiments (approximately)?:} Approximately 9-10 hours in total. 
  \item {\bf Publicly available?:} Yes.
  \item {\bf Code licenses (if publicly available)?:} Apache 2.0.
\end{itemize}
}

\subsection{Description}

\subsubsection{How delivered}

The artifact is publicly available on GitHub: \href{https://github.com/hexcute/hexcute-bench}{https://github.com/hexcute/hexcute-bench}

\subsubsection{Hardware dependencies}

An x86-64 Linux host with at least 20 CPU cores,  100 GB RAM, 100 GB free disk space, and access to both an NVIDIA A100 PCIe GPU (80 GB) and an NVIDIA H100 PCIe or SXM GPU (80 GB).

\subsubsection{Software dependencies} 

Docker is required when running inside the provided container. When running directly on the local environment, the following dependencies are required: CUDA Toolkit $\ge$ 12.6, CMake version $\ge$ 3.19, Python 3.10 (validated with 3.12), Ubuntu 22.04 (validated with 24.04).

\subsection{Installation}
\noindent 1. Clone the GitHub repository.  
{\small
\begin{lstlisting}[numbers=none,language=bash]
$ git clone https://github.com/hexcute/hexcute-bench
$ cd hexcute-bench
\end{lstlisting}
}
2. Pull the Docker image
{\small
\begin{lstlisting}[numbers=none,language=bash]
$ docker pull ghcr.io/hexcute/hexcute-bench/hexcute:latest
\end{lstlisting}
}
3. Launch the Docker container
{\small
\begin{lstlisting}[numbers=none,language=bash]
$ docker run --privileged \
    -v /path/to/hexcute-bench:/workspace/hexcute-bench \
    --gpus all -it \
    ghcr.io/hexcute/hexcute-bench/hexcute:latest /bin/bash
\end{lstlisting}
}
The \texttt{--privileged} option is required because the benchmark scripts lock the GPU frequency for reproducibility.

\subsection{Experiment workflow}
The commands below are run inside the Docker container. 

\noindent 1. Set up the environment variables. 
{\small
\begin{lstlisting}[numbers=none,language=bash]
$ export CUDA_HOME=/path/to/cuda
$ export PATH=$CUDA_HOME/bin:$PATH
\end{lstlisting}
}
2. Enter the scripts directory.

The benchmark scripts (run in Step 3) lock the GPU's frequency for reproducibility. When running inside Docker, the \texttt{--privileged} flag provides the required permissions.
{
\small
\begin{lstlisting}[numbers=none,language=bash]
$ cd /workspace/hexcute-bench/scripts
\end{lstlisting}
}
3. Run the experiments
\begin{itemize}
\item A100 kernel performance evaluation (approximately 5 hours): 
{\small
\begin{lstlisting}[numbers=none,language=bash]
$ bash run_a100.sh
\end{lstlisting}
}
\item H100 kernel performance evaluation: 
{\small
\begin{lstlisting}[numbers=none,language=bash]
$ bash run_h100.sh
\end{lstlisting}
}
\item H100 mixed-type MoE evaluation and ablation study: 
{\small
\begin{lstlisting}[numbers=none,language=bash]
$ bash run_moe.sh
\end{lstlisting}
}
\item H100 Mamba scan evaluation: 
{\small
\begin{lstlisting}[numbers=none,language=bash]
$ bash run_scan.sh
\end{lstlisting}
}
\end{itemize}

When running outside Docker (i.e., directly on the host machine), the scripts require the \texttt{--host} flag.

\subsection{Evaluation and expected results}
\noindent 1. Generate the results
{\small
\begin{lstlisting}[numbers=none,language=bash]
$ bash parse_results.sh
\end{lstlisting}
}
This script creates a directory named \texttt{hexcute\_results/} under the \texttt{hexcute-bench} directory, and moves all generated raw logs and plots into it.  

\noindent 2. Running the scripts produces the following outputs:
\begin{itemize}
\item A table equivalent to Table~\ref{general-operators}, including the geometric-mean speedups across all evaluated shapes. The table can be rendered using \texttt{pdflatex}. For example:
{\small
\begin{lstlisting}[numbers=none,language=bash]
$ pdflatex Table_II.tex
\end{lstlisting}
}
This command can be run on the host machine or inside the Docker container. 
\item The kernel performance plots corresponding to Fig.~\ref{fig:ampere_gemm}-Fig.~\ref{fig:hopper_mha} are also generated. The geometric-mean speedups reported in Table ~\ref{general-operators} are computed from these figures. 
\item Plots reproducing the results shown in Fig.~\ref{fig:moe} and Fig.~\ref{fig:ablation}, including latency curves and ablation study results.  
\item A performance plot matching Fig.~\ref{fig:selective_scan}.
\item A cost model accuracy plot corresponding to Fig.~\ref{fig:cost_model}. 
\end{itemize}






\section*{Data-Availability Statement}
Additional data related to this
publication may be found in the repository at~\cite{hexcute-artifact}.

\bibliographystyle{IEEEtran}
\bibliography{references}


\appendices
\begin{figure}[b]
    \centering
    \includegraphics[width=\linewidth]{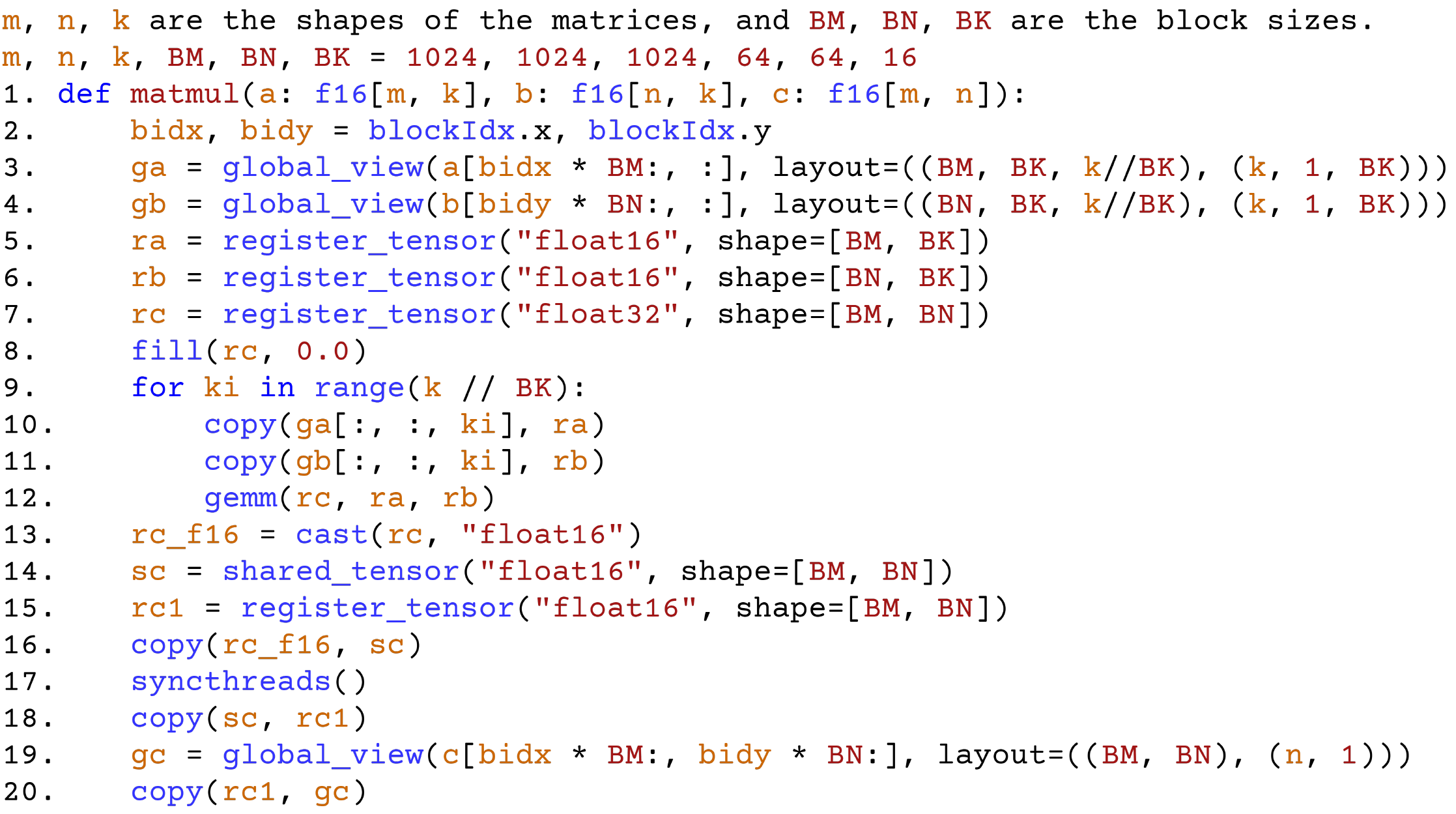}
\caption{Example: a simple GEMM kernel in Hexcute.}
\label{fig:sample_program}
\end{figure}

\section{Example: A Simple GEMM Kernel}
\label{apx:gemm}
Fig.~\ref{fig:sample_program} shows the complete FP16 \texttt{gemm} kernel used in Fig.~\ref{fig:hexcute_overview} (b), expressed with Hexcute’s tile-level primitives. Each thread block computes a $BM \times BN$ tile of the
output matrix~$c$.

\textbf{Lines} 3-4 define global memory tensors with user-specified layouts. As a kernel programming language, Hexcute requires the user to specify the layout for each global memory tensor. Specifically, \textbf{Line} 3 interprets a sub-tensor from matrix \lstinline{a} as a three-dimensional tensor of shape \( (BM, BK, k/BK) \) with strides \( (k, 1, BK) \). Similarly, \textbf{Line} 4 reshapes a sub-tensor of \lstinline{b} into shape \( (BN, BK, k/BK) \). The decomposition converts the tensor into an iterator, allowing for efficient address calculation within the loop (\textbf{Lines} 9-12). 

\textbf{Lines} 5-7 create register tensors \lstinline{ra}, \lstinline{rb}, and \lstinline{rc}, whose layouts are automatically inferred by the compiler. In \textbf{Line} 8, the accumulator tensor \lstinline{rc} is initialized to zero. The loop in \textbf{Lines} 9-12 iterates over \( k/BK \) tiles, loads tiles of \lstinline{a} and \lstinline{b} into registers (\textbf{Lines} 10-11) and performs matrix multiplication (\lstinline{gemm}) operations (\textbf{Line} 12). 

After the loop, the FP32 accumulator \lstinline{rc} is cast to FP16 in \textbf{Line} 13, stored temporarily in shared memory (\textbf{Line} 16), and read back into registers (\textbf{Line} 18).  This redistribution step ensures that threads write contiguous segments of the output tensor, enabling coalesced global-memory stores (\textbf{Line} 20).

Overall, Hexcute lets users express tile-level computation while automatically inferring register and shared-memory layouts, relieving the programmer from manual layout specification.


\section{Tile-level Operation Specifications}
\label{apx:lang_spec}
\begin{figure}[t]
{\fontsize{6.75pt}{6.5pt}\selectfont
  \begin{syntax}
    \category{\texttt{scope}}
    \alternative{\texttt{Global}}
    \alternative{\texttt{Shared}}
    \alternative{\texttt{Register}}
    \category{\texttt{type}}
    \alternative{\texttt{float16}}
    \alternative{\texttt{bfloat16}}
    \alternative{\texttt{float32}}
    \alternative{\texttt{...}}\\
    \alternative{\texttt{int8}}
    \alternative{\texttt{int32}}
    \alternative{\texttt{uint8}}
    \alternative{\texttt{uint32}}
    \alternative{\texttt{...}}\\
    \alternative{\texttt{int1}}
    \alternative{\texttt{int2}}
    \alternative{\texttt{int4}}
    \alternative{\texttt{uint1}}
    \alternative{\texttt{uint2}}
    \alternative{\texttt{uint4}}\\
    \alternative{\texttt{float8\_e5m2}}
    \alternative{\texttt{float8\_e4m3}}\\
    \alternative{\texttt{...}}
    \category{\texttt{tuple\_var}}
    \alternative{\texttt{int}}
    \alternative{\texttt{'(' int ',' ... ',' int ')'}}\\
    \alternative{\texttt{'(' tuple\_var ',' ... ',' tuple\_var ')'}}
    \category{\texttt{tuple}}
    \alternative{\texttt{'(' tuple\_var ',' ... ',' tuple\_var ')'}}
    \category{\texttt{shape}}
    \alternative{\texttt{'(' tuple ')'}}
    \category{\texttt{layout}}
    \alternative{\texttt{'(' tuple ':' tuple ')'}}
    \category{\texttt{operator}}
    \alternative{\texttt{global\_view '(' buffer ',' layout ')'}}\\
    \alternative{\texttt{register\_tensor '(' type ',' shape ')'}}\\
    \alternative{\texttt{shared\_tensor '(' type ',' shape ')'}}\\
    \alternative{\texttt{copy '(' tensor ',' tensor ')'}}\\
    \alternative{\texttt{gemm '(' tensor ',' tensor ',' tensor ')'}}\\
    \alternative{\texttt{cast '(' tensor ',' type ')'}}\\
    \alternative{\texttt{rearrange '(' tensor ',' layout ')'}}\\
    \alternative{\texttt{elementwise '(' tensor ',' ... ',' tensor ')'}}\\
    \alternative{\texttt{reduce '(' tensor ',' int ')'}}
    \category{\texttt{expr}}
    \alternative{\texttt{tensor}}\\
    \alternative{\texttt{operator}}\\
    \alternative{\texttt{expr '+' expr}}\\
    \alternative{\texttt{expr '-' expr}}\\
    \alternative{\texttt{expr '/' expr}}\\
    \alternative{\texttt{...}}
  \end{syntax}
}
  \caption{Syntax of the tile-level operations in Hexcute.}
  \label{fig:layout-syntax}
\end{figure}

This appendix describes the syntax of Hexcute's tile-level operations, shown in Fig.~\ref{fig:layout-syntax}. Hexcute provides explicit control
over memory placement (\texttt{Global}, \texttt{Shared}, \texttt{Register}) and supports low-precision integer types such as 
\texttt{int\{1,2,4\}} and \texttt{uint\{1,2,4\}}, along with 
FP8 formats \texttt{float8\_e5m2} and \texttt{float8\_e4m3}.

Following CuTe’s methodology, shapes and layouts are described using recursively defined integer tuples. A layout is written as \texttt{(tuple : tuple)}, specifying a hierarchical shape and its corresponding strides. These constructs allow users to describe sophisticated register and shared-memory layouts
compactly.

Hexcute supports tile-level primitives for data movement (\texttt{copy}), matrix multiplication (\texttt{gemm}), type conversion (\texttt{cast}), register-layout conversion (\texttt{rearrange}), elementwise computation, and reduction along a given dimension. These primitives form the core language interface of Hexcute’s tile-level programming model.

\section{Example: Computing Composite Function $g\circ q^{-1}$}
\label{apx:composite_func}
\begin{figure*}[t]
\centering
\includegraphics[width=\linewidth]{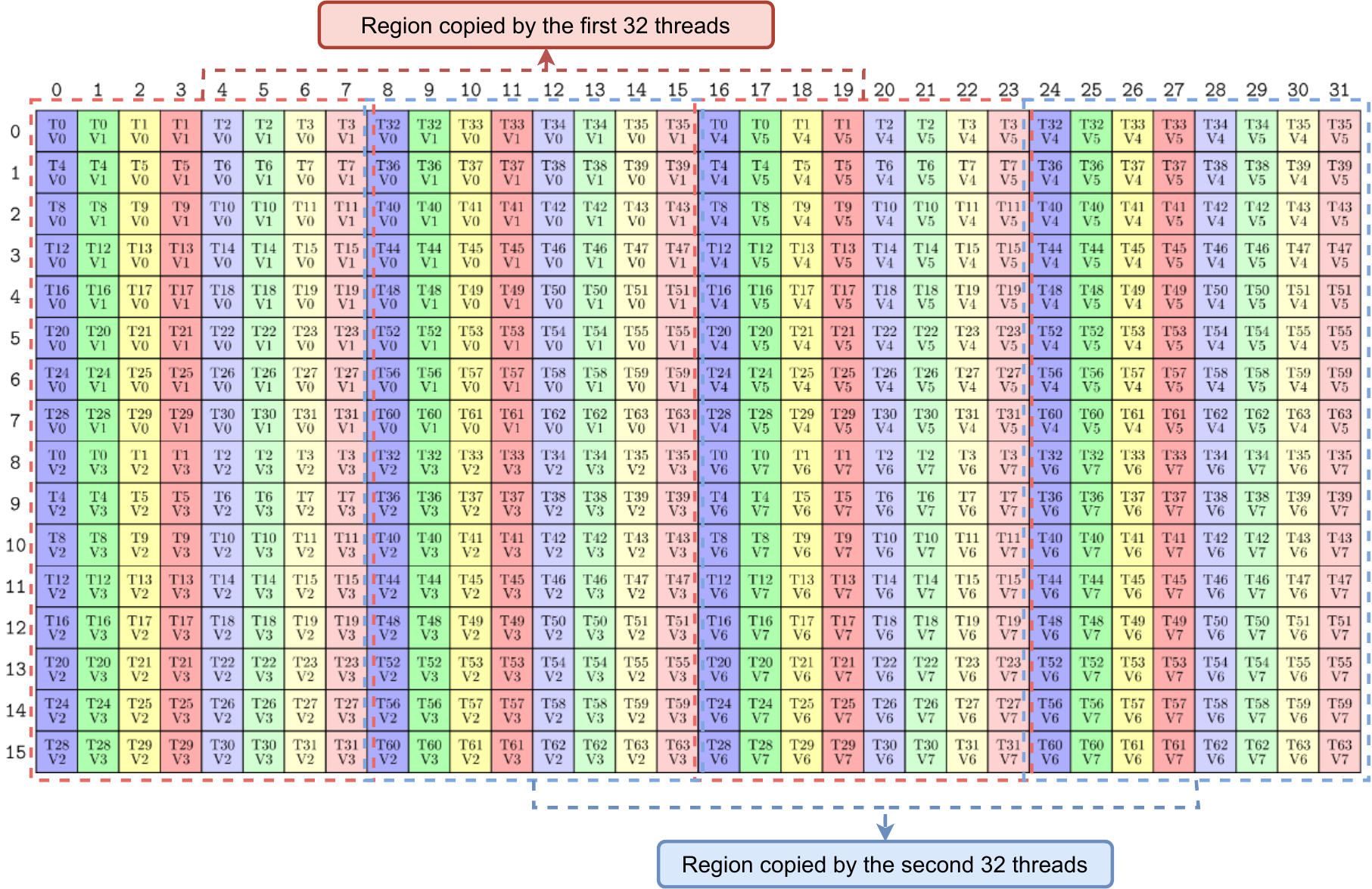}
\caption{The data distribution of register tensor \texttt{b} represented with layout $g$.}
\label{fig:layout_g}
\end{figure*}
\begin{figure*}[t]
\centering
\includegraphics[width=0.8\linewidth]{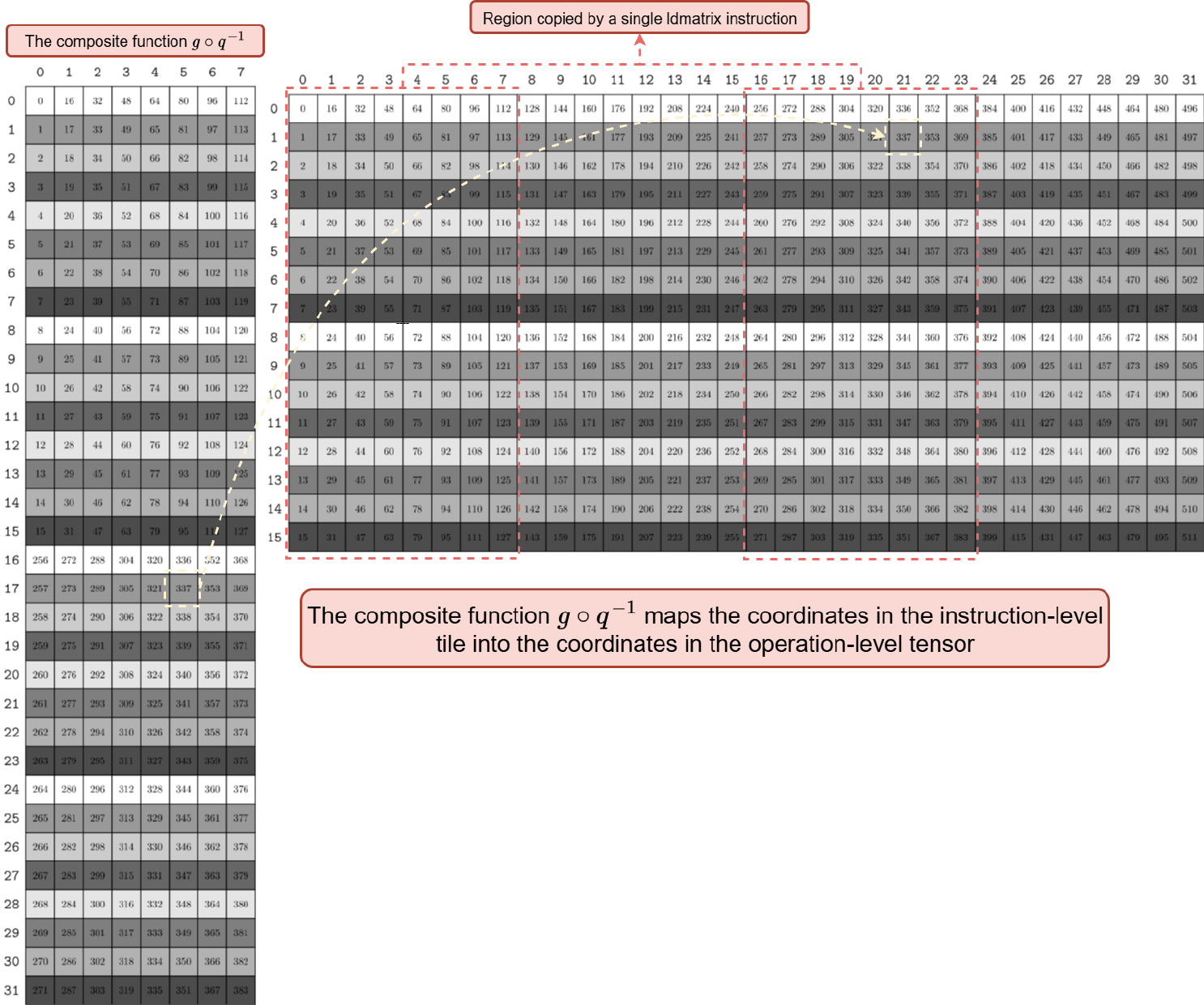}
\caption{The composite function $g\circ q^{-1}$ maps coordinates \((17, 5)\) in the instruction-level tile onto coordinates \((1, 21)\) in the operation-level tensor.}
\label{fig:mapping_g_q}
\end{figure*}
We illustrate how to compute the composite function $g \circ q^{-1}$ with a concrete example.  

Given the thread-value layout $q$ in Fig.~\ref{fig:mapping} (b), its inverse $q^{-1}$ is obtained directly using the inversion operation in CuTe~\cite{CuTe-doc}:   
\[
q^{-1} = ((8, 4), (2, 4)):((4, 64), (32, 1)).
\]  
The codomain of this function is the Cartesian product of two integer spaces, \(\mathbb{T}\times\mathbb{V}\), where \(\mathbb{T}=\{t_0,t_1,\ldots,t_{31}\}\) denotes thread indices and \(\mathbb{V}=\{v_0,v_1,\ldots,v_{7}\}\) denotes value indices.  

Consider a shared-memory \texttt{copy} that loads a $16 \times 32$ tile into registers with 64 threads. After loading, the data
is distributed according to the layout  
\[
g = ((4, 8, 2), (2, 2, 2)):((32, 1, 128), (16, 8, 256)).
\]  
Intuitively, this corresponds to 64 threads cooperatively loading the 16$\times$32 tile: every group of 32 threads loads two 16$\times$8 submatrices, as shown in Fig.~\ref{fig:layout_g}. 

To compute $g\circ q^{-1}$, we first restrict $g$ to the 32-thread subspace to align it with the codomain of $q^{-1}$:
\[
g = ((4, 8), (2, 2, 2)):((32, 1), (16, 8, 256)).
\]  
Applying the composition operation in CuTe~\cite{CuTe-doc} yields  
\[
g \circ q^{-1} = ((8, 2, 2), (2, 4)):((1, 8, 256), (16, 32)).
\]  
This resulting layout is shown on the left of Fig.~\ref{fig:mapping_g_q}, where the numbers in the boxes indicate the output of the function. The right-hand side shows the logical codomain of layout $g$, $(16, 32):(1, 16)$. Fig.~\ref{fig:mapping_g_q} allows us to verify that $g \circ q^{-1}$ correctly maps instruction-level tile coordinates to operation-level tensor coordinates.  

For example, consider input coordinates $(17, 5)$:  
\[
g \circ q^{-1}(17, 5) = 337.
\]  
Since the tensor has shape $16\times32$, this column-major order linear index corresponds to  
\begin{align*}  
r &= 337 \bmod 16 = 1, \\  
c &= 337 \div 16 = 21.  
\end{align*}  
This matches the expected two-dimensional coordinates shown in Fig.~\ref{fig:mapping_g_q}.

\section{Generalized Constraints}
\label{apx:general_constraints}
Fig.~\ref{fig:thread_value_layout_constraints} summarizes the thread-value layout constraints for the \texttt{copy}, \texttt{gemm}, \texttt{elementwise}, and \texttt{reduce} operations.

For the \texttt{copy} operation in Fig.~\ref{fig:thread_value_layout_constraints} (a), the constraint
requires the composite mappings $f \circ p^{-1}$ and $g \circ q^{-1}$ to be identical, ensuring that a single instruction moves the same logical region of data from the source tensor to the destination tensor.

For the \texttt{gemm(a, b, c)} operation in Fig.~\ref{fig:thread_value_layout_constraints} (b), the tensor operands $a$, $b$, and $c$ reside in different coordinate systems ($M{\times}K$, $N{\times}K$, and $M{\times}N$). Fig.~\ref{fig:thread_value_layout_constraints} (b) ensures that the instruction's operand tiles ($A$, $B$, $C$) must correspond consistently to these tensor-level spaces along each dimension. In Fig.~\ref{fig:thread_value_layout_constraints} (b), projection functions ($\mu_\ast$) extract one dimension of the tensor coordinates, while embedding functions ($\eta_\ast$) lift the instruction coordinates into the 2D spaces. Matching the resulting composite mappings in Fig.~\ref{fig:thread_value_layout_constraints} ensures that Tensor Core fragments align correctly with the tile-level layouts of the \texttt{gemm} operation. 

For \texttt{elementwise} operations in Fig.~\ref{fig:thread_value_layout_constraints} (c), all tensors must share the same thread-value layout.

For the \texttt{reduce} operation in Fig.~\ref{fig:thread_value_layout_constraints} (d), the output layout $g$ must match the input layout $f$ after collapsing the reduced dimension. This collapse is expressed by composing $f$ with a projection that removes the reduction axis.

\begin{figure}[t!]
\begin{prooftree}
{\small
\hypo{
\begin{tabular}{p{0.78\linewidth}}
\textbf{(a)} \lstinline[basicstyle=\bfseries\small]|copy(a, b)|\\
\text{Tensor \lstinline[basicstyle=\ttfamily\small]|a| and \lstinline[basicstyle=\ttfamily\small]|b| have thread-value layouts $f$ and $g$.}\\
\text{Instruction $I$ is represented by thread-value layouts $p$ and $q$ for}\\
\text{its input and output operands. $I$ can execute the \lstinline[basicstyle=\ttfamily\small]{copy} operation.}
\end{tabular}}
\infer1[tp/tp-copy]{f\circ p^{-1} = g\circ q^{-1}}
}
\end{prooftree}
\begin{prooftree}
{\small
\hypo{
\begin{tabular}{p{0.78\linewidth}}
\textbf{(b)} \lstinline[basicstyle=\bfseries\small]|gemm(a, b, c)|\\
\text{Tensor \lstinline[basicstyle=\ttfamily\small]|a|, \lstinline[basicstyle=\ttfamily\small]|b|, and \lstinline[basicstyle=\ttfamily\small]|c| have thread-value layouts $f_a$, $f_b$ and $f_c$.}\\
\text{Instruction $I$ is represented by thread-value layouts $p_A$, $p_B$ and $p_C$.}\\
\text{Instruction $I$ can execute the \lstinline[basicstyle=\ttfamily\small]{gemm} operation.}
\end{tabular}}
\infer1[tp/tp-gemm]{
\begin{tabular}{c}
$\mu_M\circ \left(f_c\circ p_C^{-1}\right)\circ \eta_M=\tilde{\mu}_M\circ \left(f_a\circ p_A^{-1}\right)\circ \tilde{\eta}_M$\\
$\mu_N\circ \left(f_c\circ p_C^{-1}\right)\circ \eta_N=\tilde{\mu}_N\circ \left(f_b\circ p_B^{-1}\right)\circ \tilde{\eta}_N$\\
$\mu_K\circ \left(f_a\circ p_A^{-1}\right)\circ \eta_K=\tilde{\mu}_K\circ \left(f_b\circ p_B^{-1}\right)\circ \tilde{\eta}_K$\\
\end{tabular}}
}
\end{prooftree}
\begin{prooftree}
{\small
\hypo{
\begin{tabular}{p{0.78\linewidth}}
\textbf{(c)} \lstinline[basicstyle=\bfseries\small]|elementwise(a1, a2, a3, ..., an)|\\
\text{Tensors \lstinline[basicstyle=\ttfamily\small]|a1, ..., an| have thread-value layouts $f_1, f_2, \cdots, f_n$} 
\end{tabular}}
\infer1[tp/tp-elem]{
f_1=f_2=\cdots=f_n}
}
\end{prooftree}
\begin{prooftree}
{\small
\hypo{
\begin{tabular}{p{0.76\linewidth}}
\textbf{(d)} \lstinline[basicstyle=\bfseries\small]|b = reduce(a)|\\
\text{Tensors \lstinline[basicstyle=\ttfamily\small]|a| and \lstinline[basicstyle=\ttfamily\small]|b| have thread-value layouts $f$ and $g$.} 
\end{tabular}}
\infer1[tp/tp-reduce]{\mu \circ f = g,\,\mu\,\text{collapses the reduced dimension.}}
}
\end{prooftree}
\caption{The thread-value layout constraints for \lstinline{copy}, \lstinline{gemm}, \lstinline{elementwise}, and \lstinline{reduce} operations.}
\label{fig:thread_value_layout_constraints}
\end{figure}

\section{A Formal Proof for the Thread-value Layout Constraints of \texttt{gemm} Operation}
\label{apx:proof_tv_layouts}
In this appendix, we present a formal proof of the thread-value layout constraints for the \lstinline{gemm} operation.

\begin{figure}[t]
\centering
\begin{tikzcd}[row sep=huge]
I_{M}\times I_{K} \arrow[swap]{d}{f_a\circ p_A^{-1}} & I_{M} \arrow[swap]{l}{\tilde{\eta}_M} \arrow{r}{\eta_M} \arrow{d} & I_{M}\times I_{N} \arrow{d}{f_c\circ p_C^{-1}} \\%
T_{M} \times T_{K} \arrow{r}{\tilde{\mu}_M} & T_{M} & T_{M} \times T_{N} \arrow[swap]{l}{\mu_M} 
\end{tikzcd}\\
\vspace{0.5cm}
\begin{tikzcd}[row sep=huge]
I_{N}\times I_{K} \arrow[swap]{d}{f_b\circ p_B^{-1}} & I_{N} \arrow[swap]{l}{\tilde{\eta}_N} \arrow{r}{\eta_N} \arrow{d} & I_{M}\times I_{N} \arrow{d}{f_c\circ p_C^{-1}} \\%
T_{N} \times T_{K} \arrow{r}{\tilde{\mu}_N} & T_{N} & T_{M} \times T_{N} \arrow[swap]{l}{\mu_N} 
\end{tikzcd}\\
\vspace{0.5cm}
\begin{tikzcd}[row sep=huge]
I_{M}\times I_{K} \arrow[swap]{d}{f_a\circ p_A^{-1}} & I_{K} \arrow[swap]{l}{\tilde{\eta}_K} \arrow{r}{\eta_K} \arrow{d} & I_{N}\times I_{K} \arrow{d}{f_b\circ p_B^{-1}} \\%
T_{M} \times T_{K} \arrow{r}{\tilde{\mu}_K} & T_{K} & T_{M} \times T_{K} \arrow[swap]{l}{\mu_K} 
\end{tikzcd}
\caption{The thread-value constraints for matrices \lstinline{a}, \lstinline{b}, and \lstinline{c} in the \lstinline{gemm} operation represented with commutative diagrams. The projection functions, $\mu_M$, $\tilde{\mu}_M$, $\mu_N$, $\tilde{\mu}_N$, $\mu_K$, and $\tilde{\mu}_K$, and the natural embedding functions, $\eta_M$, $\tilde{\eta}_M$, $\eta_N$, $\tilde{\eta}_N$, $\eta_K$, and $\tilde{\eta}_K$, are defined in Theorem \ref{thm:mma}. The Cartesian product $T_M\times T_N$, $T_M\times T_K$, and $T_N\times T_K$, represent the coordinate spaces of matrices, \lstinline{a}, \lstinline{b}, and \lstinline{c} in the \lstinline{gemm} operation, while the Cartesian product $I_M\times I_N$, $I_M\times I_K$, and $I_N\times I_K$, represent the coordinate spaces of $A$, $B$, and $C$ operands in instruction $I$.} 
\label{fig:commut_diag_mma}
\end{figure}
\begin{theorem}
\label{thm:mma}
Consider a \lstinline{gemm} operation where the tensors \lstinline{a}, \lstinline{b}, and \lstinline{c} have the thread-value layouts $f_a$, $f_b$, and $f_c$, respectively. Suppose we use a Tensor Core instruction $I$, represented by the thread-value layouts $p_A$, $p_B$, and $p_C$ for the $A$, $B$, and $C$ operands, to implement this operation. Let the projection be defined as follows: 
\begin{align*}
& \mu_M:(m_T, k_T)\mapsto m_T, &\tilde{\mu}_M: (m_T, n_T) \mapsto m_T \\
& \mu_N:(n_T, k_T)\mapsto n_T, &\tilde{\mu}_N: (m_T, n_T) \mapsto n_T \\
& \mu_K:(m_T, k_T)\mapsto k_T, &\tilde{\mu}_K: (n_T, k_T) \mapsto k_T
\end{align*}

The natural embedding functions are defined: 
{
\fontsize{8}{9}\selectfont
\begin{align*}
& \eta_M:m_I\mapsto (m_I, 0)\in I_M\times I_N, &\tilde{\eta}_M:m_I\mapsto (m_I, 0) \in I_M\times I_K \\
& \eta_N:n_I\mapsto (0, n_I)\in I_M\times I_N, &\tilde{\eta}_N:n_I\mapsto (n_I, 0) \in I_N\times I_K \\
& \eta_K:k_I\mapsto (0, k_I)\in I_M\times I_K, &\tilde{\eta}_K:k_I\mapsto (0, k_I) \in I_N\times I_K
\end{align*}
}

where $m_T$, $n_T$, and $k_T$ are the coordinates within tensors of the \lstinline{gemm} operation, and $m_I$, $n_I$, and $k_I$ are the coordinates within the operands of the instruction $I$. The Cartesian product $I_M\times I_N$, $I_M\times I_K$, and $I_N\times I_K$, represent the coordinate spaces of $A$, $B$, and $C$ operands of instruction $I$.
Then, the following consistency equations must hold:
\begin{align*}
& \mu_M\circ \left(f_c\circ p_C^{-1}\right)\circ \eta_M=\tilde{\mu}_M\circ \left(f_a\circ p_A^{-1}\right)\circ \tilde{\eta}_M \\
& \mu_N\circ \left(f_c\circ p_C^{-1}\right)\circ \eta_N=\tilde{\mu}_N\circ \left(f_b\circ p_B^{-1}\right)\circ \tilde{\eta}_N\\
& \mu_K\circ \left(f_a\circ p_A^{-1}\right)\circ \eta_K=\tilde{\mu}_K\circ \left(f_b\circ p_B^{-1}\right)\circ \tilde{\eta}_K\\
\end{align*}
\end{theorem}
\begin{proof}
In the \lstinline{gemm} operation, the thread-value layouts $f_a$, $f_b$, and $f_c$ of the matrices \lstinline{a}, \lstinline{b}, and \lstinline{c} map thread ID \(t\) and value ID \(v\) to the logical positions within the tensors: 
\begin{align*}
& f_a:(t, v)\mapsto(m_T, k_T)\\
& f_b:(t, v)\mapsto(n_T, k_T)\\
& f_c:(t, v)\mapsto(m_T, n_T)
\end{align*} 
Similarly, for the Tensor Core instruction $I$, represented by the thread-value layouts $p_A$, $p_B$, and $p_C$ for the $A$, $B$, and $C$ operands, these layouts map thread ID and value ID to the operand coordinates: 
\pagebreak
\begin{align*}
& p_A:(t, v)\mapsto(m_I, k_I)\\
& p_B:(t, v)\mapsto(n_I, k_I)\\
& p_C:(t, v)\mapsto(m_I, n_I)
\end{align*}

To relate the tensor layouts with the instruction layouts, we construct composite functions from the instruction operands to the matrix tensors:
\begin{align*}
& f_a\circ p_A^{-1}:(m_I, k_I)\mapsto(m_T, k_T) \\
& f_b\circ p_B^{-1}:(n_I, k_I)\mapsto(n_T, k_T) \\
& f_c\circ p_C^{-1}:(m_I, n_I)\mapsto(m_T, n_T)
\end{align*}
Now, let us consider the first equation in Theorem \ref{thm:mma}. To relate the mapping functions $f_a\circ p_A^{-1}$ (for matrix \lstinline{a}) and $f_c \circ p_C^{-1}$ (for matrix \lstinline{c}), we construct the following mappings by composing the projection and embedding functions:  
\begin{align*}
& \mu_M\circ \left(f_c\circ p_C^{-1}\right)\circ \eta_M:m_I \mapsto m_T \\
& \tilde{\mu}_M\circ \left(f_a\circ p_A^{-1}\right)\circ \tilde{\eta}_M:m_I \mapsto m_T
\end{align*}

Since both functions map from the instruction operands'(\(A\) and \(C\)) $m$-coordinate to the matrix tensors'(\lstinline{a} and \lstinline{c}) $m$-coordinate, these functions must be identical. This relationship is represented by a commutative diagram (top of Fig.~\ref{fig:commut_diag_mma}). Similarly, we can relate the mappings for matrix \lstinline{b} and matrix \lstinline{c} (the second equation in Theorem \ref{thm:mma}) and for matrix \lstinline{a} and matrix \lstinline{b} (the third equation in Theorem \ref{thm:mma}). These relationships are represented with commutative diagrams, as shown in the middle and bottom of Fig.~\ref{fig:commut_diag_mma}. 
\end{proof}

\section{Evaluation Details}
This appendix provides the detailed per-shape results that support the evaluation in Section~VII. Fig.~\ref{fig:selective_scan} shows the full
selective-scan results for the Mamba experiment in Section~\ref{evl:mamba_scan}. Figs.~\ref{fig:ampere_gemm}–Fig.~\ref{fig:hopper_mha} report the per-layer performance for all operators summarized in Table~\ref{general-operators}, including GEMM, fused MHA, and FP8 GEMM benchmarks on A100 and H100 GPUs.

\begin{figure}[t!]
    \begin{subfigure}{\linewidth}
        \centering
        \includegraphics[width=\textwidth]{Figures/selective_scan.pdf}
    \end{subfigure}
  \caption{Performance comparison of the selective scan kernels in the Mamba library and Hexcute across 20 shapes. Hexcute achieves an average speedup of 4.17$\times$ on the H100 GPU. }
\label{fig:selective_scan}
\end{figure}

\begin{figure*}[t!]
\centering
\includegraphics[width=\textwidth]{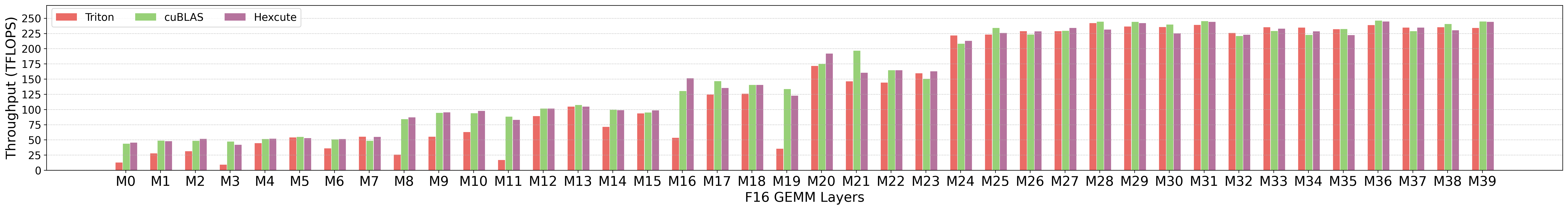}
\caption{Performance comparison of Triton, cuBLAS, and Hexcute on FP16 GEMM layers on an A100 PCIe GPU. Hexcute achieves an average speedup of 1.00$\times$ over cuBLAS and 1.33$\times$ over Triton.}  
\label{fig:ampere_gemm}
\end{figure*}

\begin{figure*}[t!]
\centering
\includegraphics[width=\textwidth]{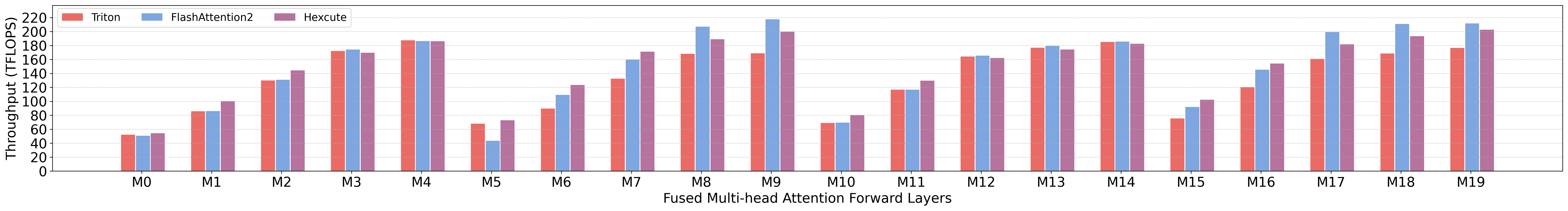}
\caption{Performance comparison of Triton, FlashAttention2, and Hexcute on fused MHA forward layers on an A100 PCIe GPU. Hexcute achieves an average speedup of 1.05$\times$ over FlashAttention2 and 1.13$\times$ over Triton.}
\label{fig:ampere_mha}
\end{figure*}

\begin{figure*}[t!]
\centering
\includegraphics[width=\textwidth]{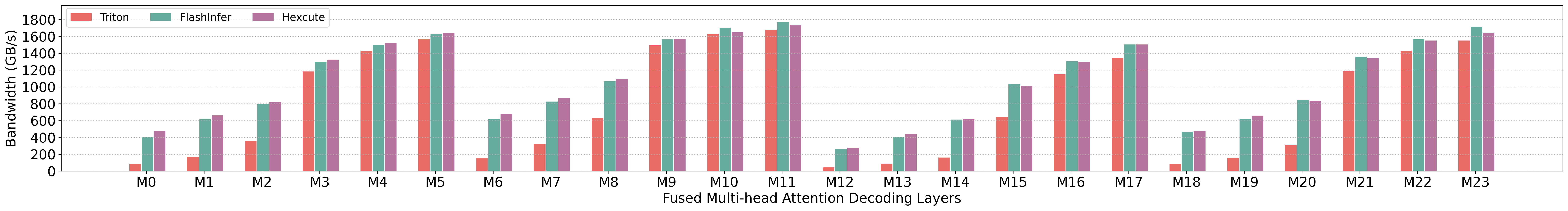}
\caption{Performance comparison of Triton, FlashInfer, and Hexcute on fused MHA decoding layers on an A100 PCIe GPU. Hexcute achieves an average speedup of 1.02$\times$ over FlashInfer and 2.06$\times$ over Triton.}
\label{fig:ampere_dec}
\end{figure*}

\begin{figure*}[t!]
\centering
\includegraphics[width=\textwidth]{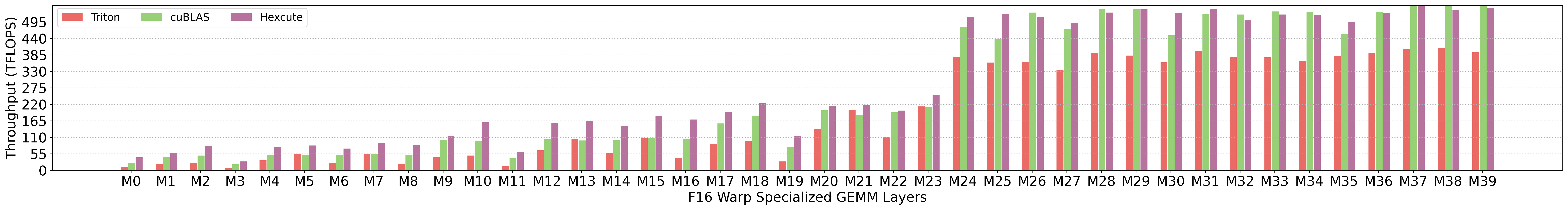}
\caption{Performance comparison of Triton, cuBLAS, and Hexcute on FP16 warp-specialized GEMM layers on an H100 PCIe GPU. Hexcute achieves an average speedup of 1.25$\times$ over cuBLAS and 1.94$\times$ over Triton.}
\label{fig:hopper_gemm}
\end{figure*}

\begin{figure*}[t!]
\centering
\includegraphics[width=\textwidth]{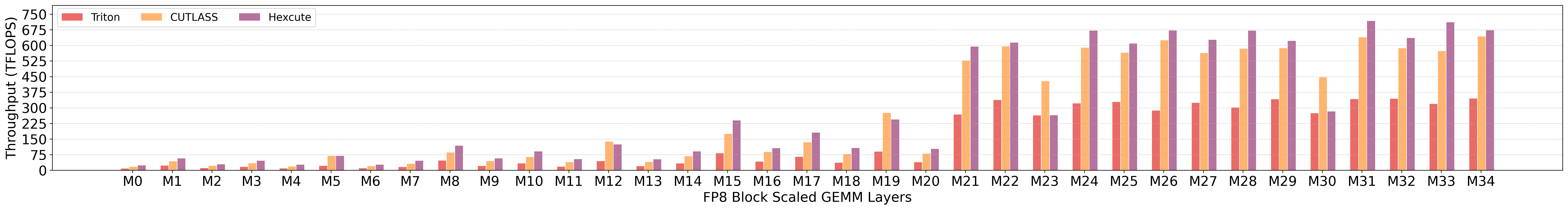}
\caption{Performance comparison of Triton, CUTLASS, and Hexcute on blockwise scaled FP8 GEMM layers on an H100 PCIe GPU. Hexcute achieves an average speedup of 1.17$\times$ over CUTLASS and 2.36$\times$ over Triton.}
\label{fig:hopper_scaled_mm}
\end{figure*}

\begin{figure*}[t!]
\centering
\includegraphics[width=\textwidth]{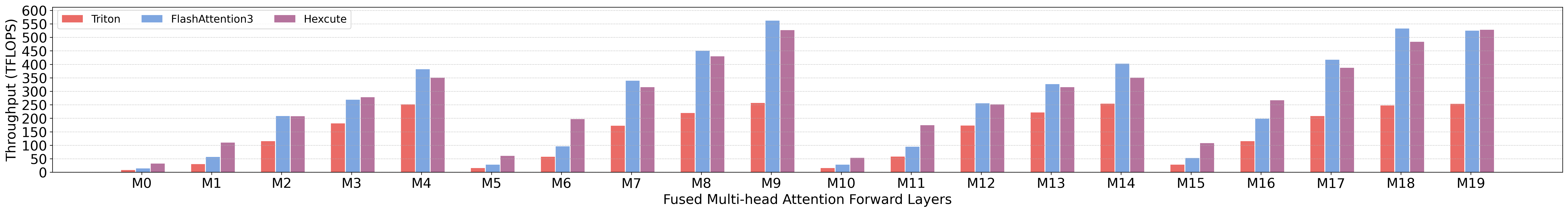}
\caption{Performance comparison of Triton, FlashAttention3, and Hexcute on fused MHA forward layers on an H100 PCIe GPU. Hexcute achieves an average speedup of 1.27$\times$ over FlashAttention3 and 2.25$\times$ over Triton.}
\label{fig:hopper_mha}
\end{figure*}

\end{document}